\pdfoutput=1

\documentclass[11pt]{article}

\usepackage[preprint]{acl}

\usepackage{times}
\usepackage{latexsym}

\usepackage[T1]{fontenc}

\usepackage[utf8]{inputenc}

\usepackage{microtype}

\usepackage{inconsolata}

\usepackage{graphicx}

\usepackage{subfigure}
\usepackage{booktabs}
\usepackage{subcaption}
\usepackage{float}
\usepackage{natbib}

\usepackage{hyperref}

\usepackage{amsmath}
\usepackage{amssymb}
\usepackage{mathtools}
\usepackage{amsthm}

\usepackage{titlesec}

\makeatletter
\renewcommand{\tableofcontents}{%
    \section*{\contentsname}%
    \begingroup
    \setlength{\baselineskip}{15pt} 
    \@starttoc{toc}%
    \endgroup
}
\makeatother

\titleformat{\paragraph}[runin] 
{\normalfont\normalsize\bfseries}{\theparagraph}{0.3em}{}
\titlespacing{\paragraph}{0pt}{0.7ex}{0.5em}

\usepackage[capitalize,noabbrev]{cleveref}

\newtheorem{theorem}{Theorem}[section]
\newtheorem{corollary}{Corollary}[section]

\newtheorem{definition}{Definition}[section]
\newtheorem{assumption}{Assumption}[section]
\newtheorem{claim}{Claim}[section]
\newtheorem{lemma}{Lemma}[section]
\newtheorem*{remark}{Remark}

\newcommand{\R}{\mathbb{R}}
\newcommand{\E}{\mathbb{E}}
\newcommand{\cX}{\mathcal{X}}
\newcommand{\cY}{\mathcal{Y}}

\newcommand{\cC}{\mathcal{C}}

\newcommand{\expect}{\mathbb{E}}
\newcommand{\prob}{\mathbb{P}}
\newcommand{\cO}{\mathcal{O}}
\newcommand{\cH}{\mathcal{H}}
\newcommand{\cF}{\mathcal{F}}
\newcommand{\cP}{\mathcal{P}}
\newcommand{\bE}{\mathbb{E}}

\newcommand{\argmin}{\text{argmin}}
\newcommand{\eps}{\varepsilon}
\newcommand{\kl}{\mathrm{d}_{\mathrm{KL}}}
\newcommand{\tv}{\mathrm{D}_{\mathrm{TV}}}

\newcommand{\is}{\mathrm{D}_{\mathrm{IS}}}
\newcommand{\wis}{\mathrm{D}_{\mathrm{WIS}}}
\newcommand{\dist}{d_\mathcal{P}}
\newcommand{\disthat}{\hat{d}_\mathcal{P}}

\allowdisplaybreaks[4]

\usepackage{url}

\title{The Capabilities and Limitations of Weak-to-Strong Generalization: Generalization and Calibration}

\author{
\textbf{Wei Yao\textsuperscript{1{$\star$}}},
\textbf{Wenkai Yang\textsuperscript{1{$\star$}}},
\textbf{Gengze Xu\textsuperscript{1}},
\textbf{Ziqiao Wang\textsuperscript{2}},
\textbf{Yankai Lin\textsuperscript{1}},
\textbf{Yong Liu\textsuperscript{1}$^{\dag}$}
\\
\\
 \textsuperscript{1}Gaoling School of Artificial Intelligence, Renmin University of China,
\\
 \textsuperscript{2}School of Computer Science and Technology, Tongji University,
\\
\tt\footnotesize\{wei.yao, wenkaiyang, liuyonggsai\}@ruc.edu.cn\\
}

\begin{document}
\maketitle

\let\thefootnote\relax\footnotetext{$^\star$ Equal contribution\hspace{3pt} \hspace{5pt}$^{\dag}$ Corresponding author\hspace{5pt}}

\begin{abstract}
Weak-to-strong generalization, where weakly supervised strong models outperform their weaker teachers, offers a promising approach to aligning superhuman models with human values. 
To deepen the understanding of this approach, we provide theoretical insights into its capabilities and limitations. 
First, in the classification setting, we establish upper and lower generalization error bounds for the strong model, identifying the primary limitations as stemming from the weak model's generalization error and the optimization objective itself.
Additionally, we derive lower and upper bounds on the calibration error of the strong model. 
These theoretical bounds reveal two critical insights: (1) the weak model should demonstrate strong generalization performance and maintain well-calibrated predictions, and (2) the strong model's training process must strike a careful balance, as excessive optimization may lead to overfitting to the weak supervision.
Finally, in the regression setting, we theoretically extend the work of~\citet{charikar2024quantifying} to a loss function based on KL divergence, offering guarantees that the strong student can outperform its weak teacher by at least the magnitude of their disagreement. 
The theory is validated through synthetic experiments.
\end{abstract}

\section{Introduction}

Human supervision~\citep{ouyang2022training,bai2022training} plays a crucial role in building both effective and safe artificial intelligence systems~\citep{achiam2023gpt,touvron2023llama}.
However, as future superhuman models exhibit increasingly complex behaviors, reliable human oversight becomes increasingly challenging~\citep{openai_superalignment}.

To tackle this issue, the Weak-To-Strong Generalization (W2SG) paradigm~\citep{burns2023weak} is proposed.
It finds that, strong pre-trained language models, when fine-tuned using labels produced by weaker models, consistently achieve better performance than their weak supervisors.
This intriguing phenomenon has not only driven the development of diverse alignment algorithms~\citep{zhu2024weak,liu2024co}, but also inspired efforts~\citep{pawelczyk2024generalizing,yang-etal-2024-weak,guo2024vision} to extend the concept to other tasks.
However, despite its empirical success, the theoretical foundations of W2SG remain underdeveloped.
Although several elegant theoretical frameworks~\citep{lang2024theoretical,somerstep2024statistical,wu2024provable,charikar2024quantifying} are proposed, a universal framework is still lacking to address fundamental questions, such as: \textit{What is the optimal generalization performance a strong model can achieve after W2SG? What other factors are influenced by W2SG?}

To answer these questions, we provide a theoretical analysis of W2SG, shedding lights on its capabilities and limitations.
Firstly, in classification tasks, 
our theoretical analysis of lower and upper generalization bounds under KL divergence loss reveals that the strong model's performance is fundamentally determined by two key factors:
(1) the disagreement between strong and weak models, which serves as the minimization objective in W2SG, and (2) the weak model's performance.
These findings suggest that (1) achieving the minimal optimization objective in W2SG limits the strong model’s ability to significantly outperform its weak supervisor, and (2) selecting a stronger weak model can enhance the performance of the strong model.
Secondly, we investigate how strong model's calibration—the property that a model's predicted confidence aligns with its actual accuracy~\citep{guo2017calibration,kumar2019verified}—is affected in the W2SG framework.
Our theoretical bounds reveal that the calibration of the strong model depends on both the calibration of the weak model and the disagreement between the two models.
The theory highlights the importance of avoiding a poorly-calibrated weak model and an overfitted strong model.
The above theoretical analysis is validated using GPT-2 series~\citep{radford2019language} and Pythia series~\citep{biderman2023pythia}.

In addition to classification setting, we also consider the regression problem.
In particular, we build on~\citet{charikar2024quantifying} by extending their analysis of squared loss to output distribution divergence, a measure of the difference between two models' output distributions. 
In this setting, the model outputs are normalized to form valid probability distributions over all input data, and the output distribution divergence between two models is defined as the KL divergence of their respective output distributions.
We recover the findings from~\citet{charikar2024quantifying} and show that the strong model's generalization error is provably smaller than the weak model's, with the gap no less than the W2SG minimization objective—namely, the strong model's error on the weak labels.
We conduct synthetic experiments to support our theoretical insights.

\section{Related Work}

In this section, we introduce AI alignment and W2SG.
Additional related work including teacher-student learning paradigm, weakly-supervised learning, calibration and information-theoretic analysis is provided in~\cref{appendix:related_work}.

\paragraph{AI alignment.}
AI alignment~\citep{ji2023ai,shen2023large} aims to ensure AI systems act in accordance with human values. 
A popular approach to achieve this goal is fine-tuning models on human-annotated data, such as Reinforcement Learning from Human Feedback (RLHF)~\citep{ouyang2022training,bai2022training} and Direct Preference Optimization (DPO)~\citep{rafailov2024direct}. 
However, this alignment paradigm faces significant challenges: human oversight becomes insufficient as AI surpasses human capabilities~\citep{kim2024road}, and obtaining scalable, high-quality human feedback remains difficult~\citep{casper2023open}. 
These challenges highlight the critical need to align superhuman AI systems~\citep{openai_superalignment}.
In contrast to these approaches, our work explores W2SG, which 
leverages weak models as the teacher to achieve the alignment goal.

\paragraph{Weak-to-strong generalization.}
To explore the effect of weak models to supervise strong models, \citet{burns2023weak} first find that strong models supervised by weak models can exhibit better performance on corresponding tasks than their weak supervisors, indicating the possibility of stimulating greater power from super models under weak supervisions.
There are also algorithms~\citep{zhu2024weak,agrawal2024ensemw2s,sang2024improving,guo2024improving} and empirical analysis~\citep{yang2024super,ye2024weak} for it.
However, only a limited number of theoretical studies have been conducted on this topic.
\citet{lang2024theoretical} analyzes it by introducing theoretical bounds that account for pseudolabel correction and coverage expansion. \citet{somerstep2024statistical} frame W2SG as a transfer learning problem, revealing limitations of fine-tuning on weak labels. \citet{wu2024provable} study linear models under a spiked covariance setting and derive asymptotic bounds. \citet{charikar2024quantifying} take a convex-theoretic approach in regression, quantifying performance improvements under squared loss via the misfit error between weak and strong models.
The work most closely related to ours is~\citet{charikar2024quantifying}, which primarily focuses on squared loss in regression. 
In contrast, we consider KL divergence-like losses, including KL divergence for classification and output distribution divergence for regression. 
Furthermore, while they focuses on establishing upper bounds, our study incorporates both upper and lower bounds as well as calibration analysis through experiments on language models, providing a more comprehensive understanding of the fundamental capabilities and limitations of W2SG.

\section{Preliminaries}

\subsection{Classification and Regression} \label{prelim_class_regress}

We examine two problem settings.
In the first case, we consider classification tasks using KL divergence as the loss function. Minimizing this loss is equivalent to minimizing cross-entropy loss, which is widely used in the W2SG literature~\citep{burns2023weak}. 
In the second case, we focus on regression tasks, employing the KL divergence between the predictions of two models as the loss function. 
The model outputs over the entire data domain are normalized to form probability distributions.
This approach is an extension of previous result~\citep{charikar2024quantifying} on squared loss, and provides an intuitive framework for understanding W2SG.

Given the data distribution $\cP$, data domain $\cX$ and output domain $\cY$, let $\cF: \cX \to \cY$.
Consider the difference $\dist$ and empirical difference $\disthat$ between two models, where $\dist,\disthat: \cF \times \cF \to \R_0^+$. We define the below two settings:

\paragraph{Setting 1: KL divergence loss.}
Firstly, we consider a $k$-classification problem.
Given the data domain $\cX \subseteq \R^d$ and output domain $\cY \subseteq \R^k$. Consider the model with the softmax module, i.e., $\forall y = (y_1, \cdots, y_k)^T \in \cY$, there holds $\sum_{i=1}^k y_i=1$ and $0 < y_i \le 1$.
Given two models $f,g \in \cF$,
define $\dist$ and $\disthat$:
\begin{align}
& \dist(f,g) \triangleq \bE_{x \sim \cP} \left[ \mathrm{D}_{\mathrm{KL}}(f(x) \| g(x)) \right], \\ \label{def:cross_entropy}
& \disthat(f,g) \triangleq \frac{1}{n} \sum_{j=1}^n \mathrm{D}_{\mathrm{KL}}(f(x_j) \| g(x_j)),  
\end{align}
where the KL divergence $\mathrm{D}_{\mathrm{KL}}(f(x) \| g(x)) = \sum_{i=1}^k [f(x)]_i \log \frac{[f(x)]_i}{[g(x)]_i}$ is the KL divergence between predictions, and $[f(x)]_i, [g(x)]_i$ are elements of $f(x), g(x)$. 

\paragraph{Setting 2: Output distribution divergence.} 
While KL divergence loss serve as de facto standard for classification, recent research demonstrates that KL divergence has been effectively employed in regression~\citep{imani2018improving,yang2021learning,gunder2022full,kitazawa2025bounds}.
Therefore, we also consider a regression problem using KL divergence in W2SG.
Let the data domain and output domain be $\cX \subseteq \R^d$ and $\cY = \{ y \in \R| 0 < y \le 1 \}$, respectively. 
In this setting, the outputs of the model for all input data are probability-normalized to ensure they form valid probability distributions. The difference between two models $f,g \in \cF$ is then measured as the KL divergence between their corresponding output distributions:
\begin{align}
& \dist(f,g) \triangleq \int_{\cX} f(x) \log \frac{f(x)}{g(x)} d x, \label{def:kl_dist} \\
& \disthat(f,g) \triangleq \sum_{i=1}^n f(x_i) \log \frac{f(x_i)}{g(x_i)}. \label{def:kl_dist_emp}
\end{align}


\subsection{Weak-to-strong Generalization}

In the context of W2SG, we focus on the fine-tuning phase after pre-training.
Let $h^\star:\R^d \to \R^{d^\star}$ denotes the ground truth representation function, which maps data $x \in \cX$ to an ideal, fully enriched representation $h^\star(x)$. 
The target fine-tuning task, composed with the ground truth representation, is denoted as $f^\star \circ h^\star$, where $f^\star:\R^{d^\star}\to \cY$.
The \textit{weak model} learns a mapping $f_w \circ h_w$, where the pre-trained representation $h_w:\cX \to \R^{d_w}$ extracts features from the input data, and $f_w:\R^{d_w} \to \cY$ is fine-tuned using supervised data with ground truth labels.
The strong model, on the other hand, aims to learn a mapping $f_{sw} \circ h_s$, where $h_s: \cX \to \R^{d_s}$ is the representation, and $f_{sw} \in \cF_{s}$ is a task-specific function from a hypothesis class $\cF_{s}: \R^{d_s} \to \cY$.
The strong model leverages the representation $h_s$ to improve performance on the fine-tuning task.
In the convention setting of AI alignment~\citep{ouyang2022training}, 
the model is learned through human-annotated ground truth data:
\begin{align}
    \label{eqn:alignment-population-minimizer}
    f_{s} = \argmin_{f \in \cF_{s}}\; \dist(f^\star \circ h^\star, f \circ h_s).
\end{align}
Nevertheless, the acquisition of human-generated data is both costly and time-consuming.
To address this challenge, the W2SG framework leverages weak supervision from the weak model's predictions, enabling the strong model to be trained through population risk minimization:
\begin{align}
    \label{eqn:fsw-population-minimizer}
    f_{sw} = \argmin_{f \in \cF_{s}}\; \dist(f_w \circ h_w, f \circ h_s).
\end{align}
In practice, we label $n$ i.i.d. samples using the weak model and minimize the empirical risk:
\begin{align} \label{eqn:erm}
    \hat{f}_{sw} = \argmin_{f \in \cF_s} \disthat(f_w \circ h_w, f \circ h_s).
\end{align}

Denote the labeling function $F^\star=f^\star \circ h^\star$, strong ceiling model $F_s=f_s \circ h_s$, weak model $F_w=f_w \circ h_w$, and strong models $F_{sw}=f_{sw} \circ h_s$, $\hat{F}_{sw}=\hat{f}_{sw} \circ h_s$, respectively.


\section{Universal Results in W2SG} \label{section:universal_result}

In this section, we consider the classification problem, where $\dist$ is the KL divergence loss defined in~\cref{def:cross_entropy}.
We first establish lower and upper generalization error bounds of the strong model in W2SG in~\cref{section_lower_upper}.
Then the lower and upper calibration error bounds are shown in~\cref{subsec:suff_nece_condition}.

\subsection{Generalization Error Bounds} \label{section_lower_upper}

\begin{theorem}[Proved in \cref{proof_lemma_inf}] \label{lemma:upper_lower_inf}
Given the data domain $\cX$, output domain $\cY$ and models $F_{sw}, F_w, F^\star$ defined above. 
Then there holds
\begin{multline*}
    \left| \dist(F^\star, F_{sw}) - \dist\left( F^\star, F_w \right) \right| \\ \le \cO \left(\sqrt{\dist(F_{sw}, F_w)} \right),
\end{multline*}
\end{theorem}

\begin{remark}
    The proof can be also extended to the regression setting, which is provided in~\cref{proof:lower_upper}.
\end{remark}

\cref{lemma:upper_lower_inf} provides a quantitative framework for assessing the performance gap between weak model and strong model in W2SG.
Specifically, the value of $\dist(F^\star, F_{sw})$ is constrained by two terms: 
(1) $\dist(F^\star, F_{sw})$, which reflects the performance of the weak model, and 
(2) $\dist(F_w, F_{sw})$, which is decided by the optimization result in~\cref{eqn:fsw-population-minimizer} and measures how the strong model learns to imitate the weak supervisor.
This result is examined from two complementary perspectives: a lower bound and an upper bound.

\paragraph{Lower bound.}
The lower bound indicates the fundamental limitation: $\dist(F^\star, F_{sw})$ cannot be arbitrarily reduced.
Firstly, a minimal $\dist(F^\star, F_{sw})$ is intrinsically tied to the weak model performance $\dist\left( F^\star, F_w \right)$.
To improve the strong model, the weak model becomes critical---that is, $\dist\left( F^\star, F_w \right)$ should be as small as possible. It underscores the importance of \textbf{\textit{carefully selecting the weak model}}~\citep{burns2023weak,yang2024super}.
Secondly, the performance improvement of strong model over the weak model cannot exceed $\cO \left(\sqrt{\dist(F_w, F_{sw})} \right)$.
In W2SG, while the student-supervisor disagreement $\dist(F_w, F_{sw})$ is minimized in~\cref{eqn:fsw-population-minimizer}, we anticipate $\cO \left(\sqrt{\dist(F_w, F_{sw})} \right)$ to remain relatively small.
However, a paradox arises: achieving a smaller $\dist\left( F^\star, F_{sw} \right)$ necessitates a larger $\dist(F_w, F_{sw})$.
This implies that \textbf{\textit{the performance improvement of W2SG is probably constrained by its own optimization objective}}.

\paragraph{Upper bound.}
The upper bound provides a theoretical guarantee for W2SG by ensuring that $\dist(F^\star, F_{sw})$ remains bounded and does not grow arbitrarily large.
Firstly, effective W2SG requires choosing a weak model that produces supervision signal closely aligned with the true score, i.e., achieving a small $\dist\left( F^\star, F_w\right)$. 
To this end, employing a stronger weak model is crucial to obtain a tighter upper bound of $\dist(F^\star, F_{sw})$.
Secondly, the worst-case performance of the strong model is constrained by the sum of $\dist\left( F^\star, F_w\right)$ and $\cO \left(\sqrt{\dist(F_w, F_{sw})} \right)$.
By appropriately selecting the weak model and determining the minimizer of~\cref{eqn:fsw-population-minimizer}, both $\dist\left( F^\star, F_w\right)$ and $\cO \left(\sqrt{\dist(F_w, F_{sw})} \right)$ can be kept small, ensuring the practicality of the strong model.

\subsection{Calibration Error Bound} \label{subsec:suff_nece_condition}

In this section, we further explore W2SG through the lens of calibration~\citep{kumar2019verified}, which requires that the prediction confidence should match the actual outcome.
We first state the definition of Marginal Calibration Error (MCE)~\citep{kumar2019verified}, which is an extended version of Expected Calibration Error (ECE)~\citep{guo2017calibration} designed for multi-class classification.
In particular, we use an $\ell_1$ version of MCE, with the weight constant $\frac{1}{k}$ omitted.
\begin{definition}[Marginal Calibration Error~\citep{kumar2019verified}] \label{def:cal:mce}
Let $x \in \cX$, ground truth $y=[y_1, \cdots, y_k]^T \in \{ 0,1 \}^k$ where $\sum_{i=1}^k y_i=1$, and a model $f: \cX \to \cY$.
Define the marginal calibration error of $f$ as:
\vspace{-2pt}
\begin{align} \label{def:cal_err}
    \textit{MCE}(f) = \sum_{i=1}^k \expect_x \left| [f(x)]_i-\prob[y_i=1|[f(x)]_i] \right|.
\end{align}
\end{definition}

\vspace{-7pt}
It measures the difference between model confidence and actual outcome, and $\textit{MCE}(f) \in [0,2]$.
For binary classification, $\textit{MCE}$ is twice the $\textit{ECE}$.
We shed light on upper and lower bounds of calibration of the strong model.

\begin{theorem}[Proved in~\cref{proof:calibration}] \label{theorem:calibration}
Let $\text{MCE}(\cdot)$ be the marginal calibration error in~\cref{def:cal:mce}.
Then there holds
\vspace{-7pt}
\begin{multline}
    \left| \textit{MCE}(F_{sw}) - \textit{MCE}(F_w) \right| \\ \le 2 \cdot \sqrt{1-\exp{\left(-\dist(F_w,F_{sw})\right)}}.
\end{multline}
\end{theorem}

\vspace{-4pt}

\cref{theorem:calibration} demonstrates that the calibration error of $F_{sw}$ is influenced by two key factors: (1) the calibration error of $F_w$, and (2) the teacher-student disagreement, as characterized by the optimization result in~\cref{eqn:fsw-population-minimizer}.
This theoretical result yields two insights.
First, to achieve a strong model with acceptable calibration, the weak teacher should also exhibit acceptable calibration. Otherwise, the strong model will inherit a non-trivial calibration error from the weak teacher as $\dist(F_w,F_{sw})$ goes to zero.
Second, closely imitating the weak supervisor minimizes $\dist(F_w,F_{sw})$, causing the calibration errors of the strong and weak models to converge. 
Taking them together, to ensure W2SG with reasonable calibration and prevent a poorly-calibrated $F_{sw}$, it is crucial to \textit{\textbf{avoid using a poorly-calibrated weak model with an overfitted strong model}}.
Additionally, since models with larger capacity may exhibit higher calibration errors~\citep{guo2017calibration}, a potential trade-off may exist between the weak model's calibration error and the teacher-student disagreement.
In other words, $\textit{MCE}(F_w)$ and $\sqrt{1-\exp{\left(-\dist(F_w,F_{sw})\right)}}$ may not be minimized simultaneously, posing a challenge in selecting the weak model and designing an effective optimization strategy to achieve better calibration in the strong model.

\subsection{Experimental Validation in Language Models}

\begin{figure*}[t]
  \centering
  \subfigure[Accuracy (Pythia).]{
    \includegraphics[width=0.47\textwidth]{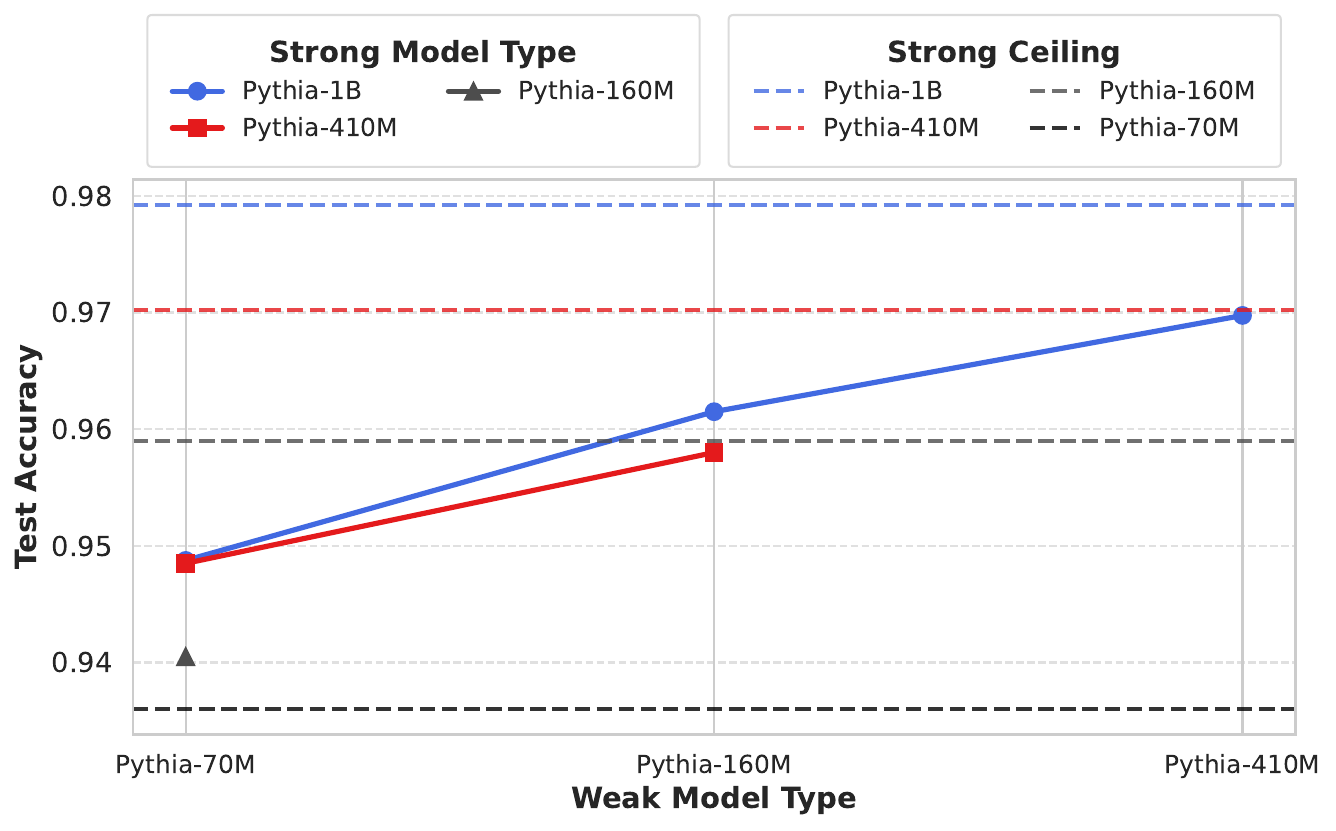}
    \label{fig:pythia_acc}
  }
  \subfigure[Accuracy (GPT-2).]{
    \includegraphics[width=0.47\textwidth]{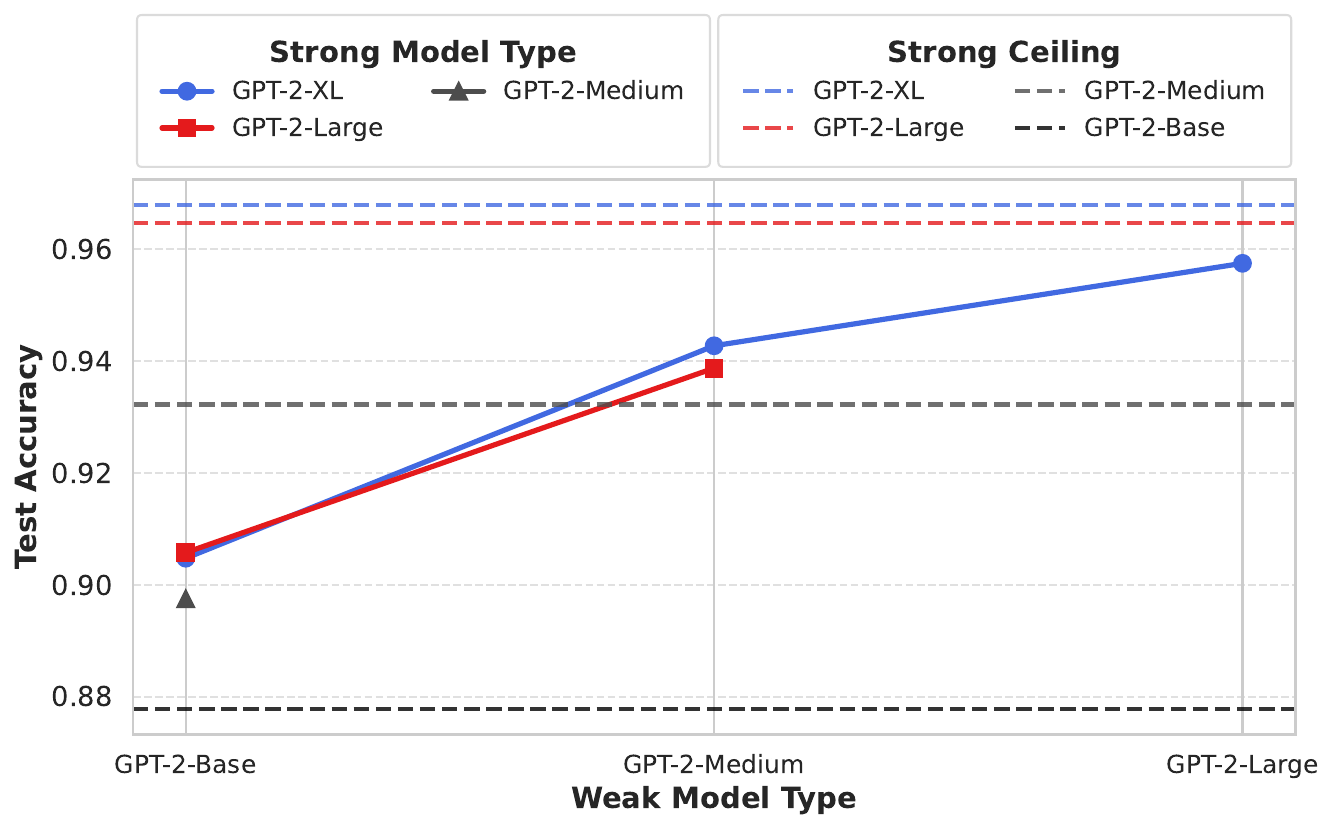}
    \label{fig:gpt_acc}
  }
  \subfigure[ECE (Pythia).]{
    \includegraphics[width=0.47\textwidth]{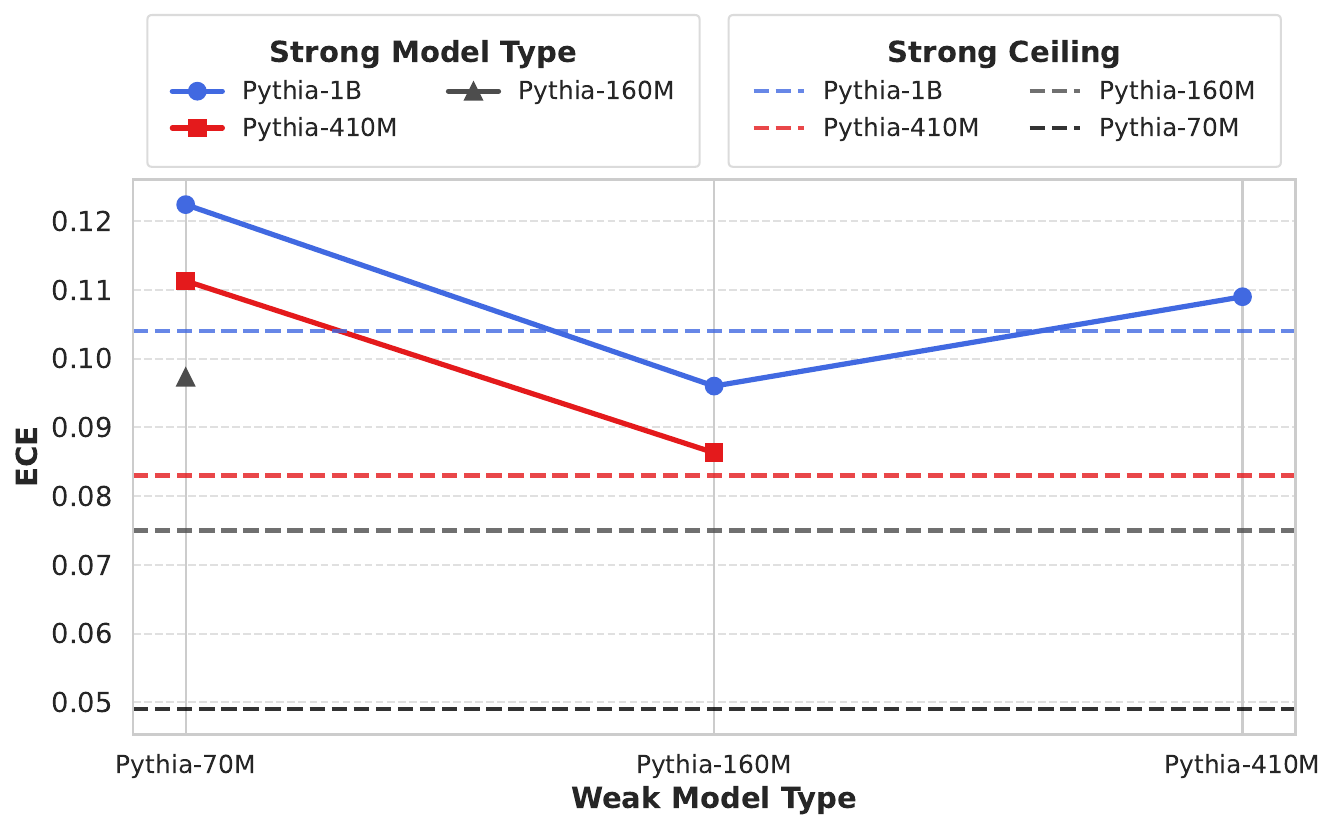}
    \label{fig:pythia_calibration}
  }
  \subfigure[ECE (GPT-2).]{
    \includegraphics[width=0.47\textwidth]{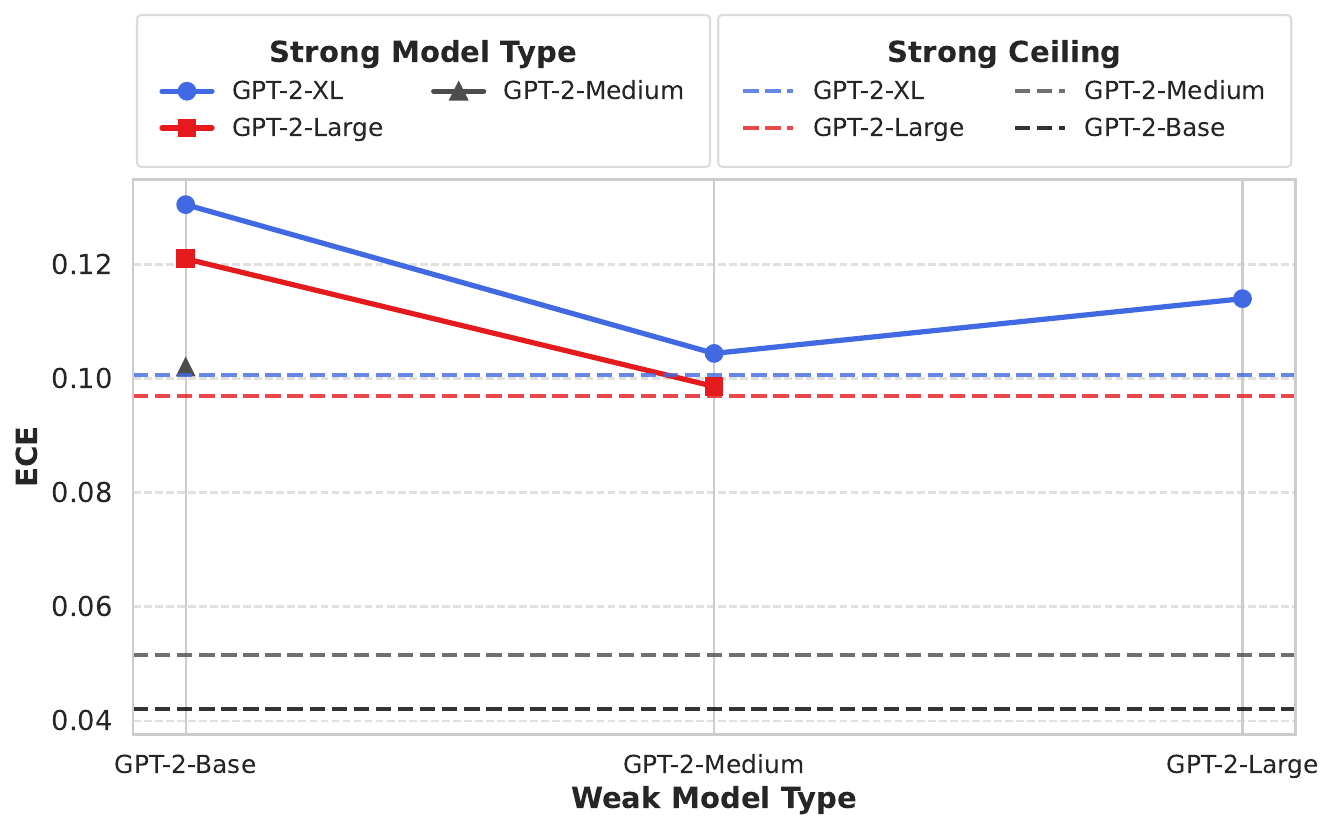}
    \label{fig:gpt_calibration}
  }
  \vspace{-5pt}
  \caption{Accuracy and calibration results for Pythia and GPT-2 series. (\textbf{a}) Test accuracies of Pythia series. (\textbf{b}) Test accuracies of GPT-2 series. 
  Each curve demonstrates the variation in accuracy of W2SG as strong models are supervised by weak models of varying capabilities. ``Strong Ceiling'' corresponds to models fine-tuned using ground truth data. 
  (\textbf{c}) Expected calibration errors of Pythia series. Each curve depicts the change in ECE as strong models are supervised by different weak teachers. (\textbf{d}) Expected calibration errors of GPT-2 series.
  }
  \label{exp_llm_main}
  \vspace{-8pt}
\end{figure*}

\begin{figure*}[t]
  \centering
  \subfigure[Accuracy (Pythia).]{
    \includegraphics[width=0.23\textwidth]{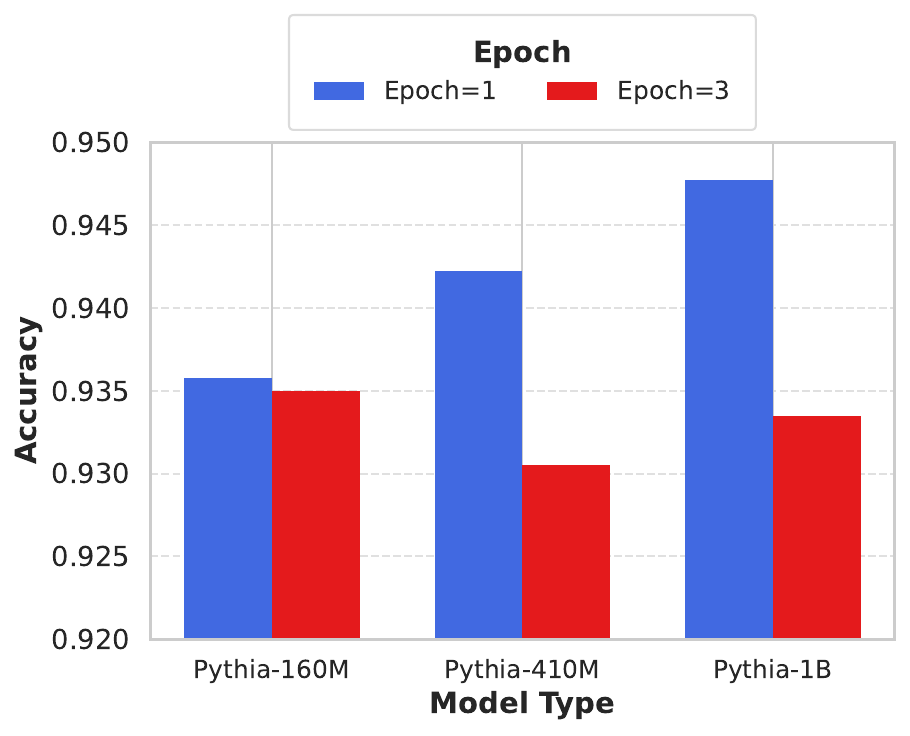}
  }
  \subfigure[Accuracy (GPT-2).]{
    \includegraphics[width=0.23\textwidth]{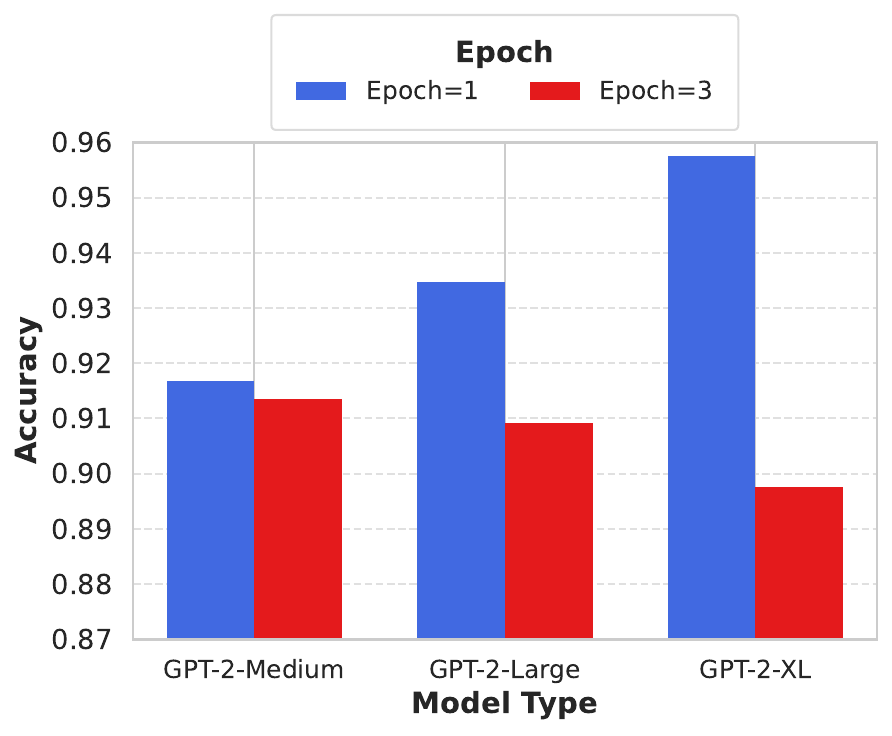}
  }
  \subfigure[ECE (Pythia).]{
    \includegraphics[width=0.23\textwidth]{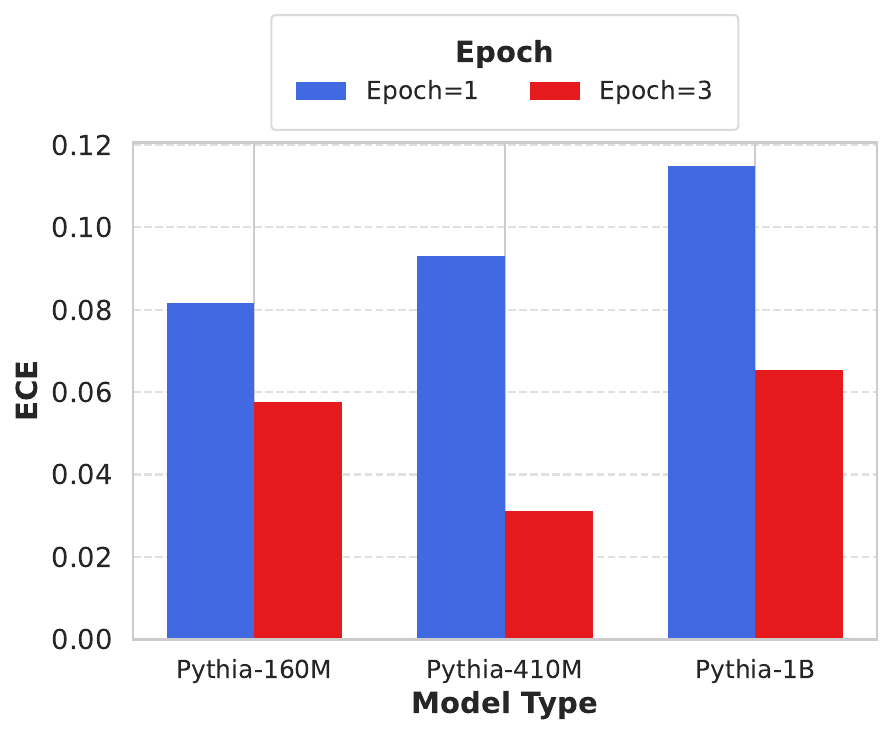}
  }
  \subfigure[ECE (GPT-2).]{
    \includegraphics[width=0.23\textwidth]{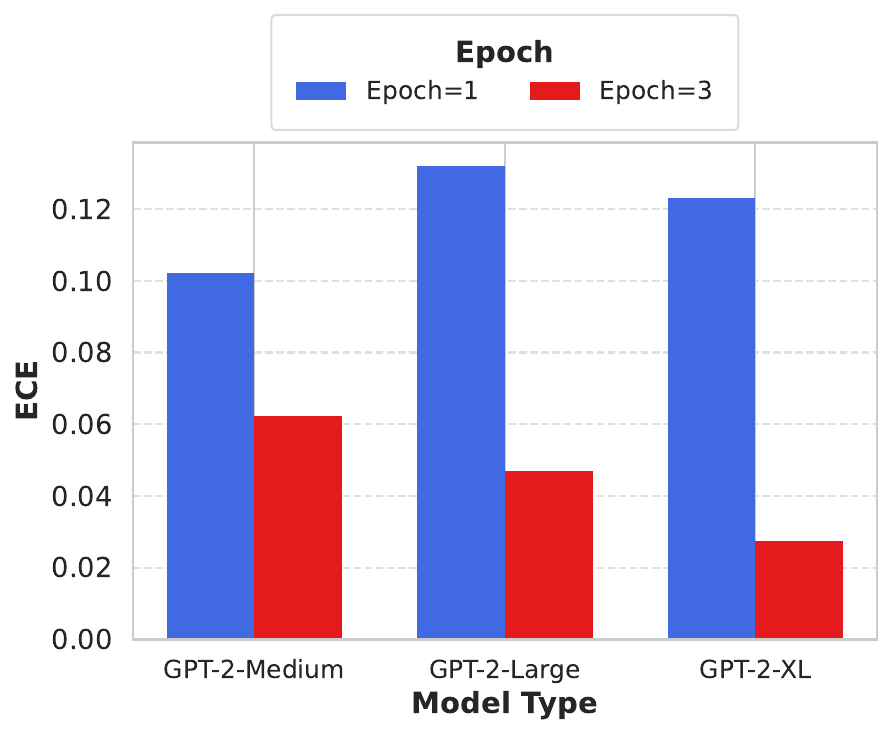}
  }
  \vspace{-7pt}
  \caption{Ablation study for the Pythia and GPT-2 series. (\textbf{a})-(\textbf{b}) Test accuracies of Pythia and GPT-2. The accuracies of Pythia-70M and GPT-2-Base fine-tuned on ground truth data is 92.45\% and 90.95\%, respectively.
  (\textbf{c})-(\textbf{d}) ECE of Pythia and GPT-2. The ECE of Pythia-70M and GPT-2-Base fine-tuned on ground truth data is 0.049 and 0.042, respectively.}
  \label{exp_llm_ablation}
  \vspace{-10pt}
\end{figure*}

In this section, we use language models to verify our theoretical results in W2SG.

\subsubsection{Experimental Setting}

\paragraph{Dataset.}
We define the alignment objective as enabling a weak model to guide a strong model in achieving harmlessness. To this end, we employ CAI-Harmless~\citep{bai2022constitutional}, which is a widely adopted single-turn harmless dataset for reward modeling task. 
Each sample is structured as $(x;y_c,y_r)$, where $x$ denotes the prompt, and $y_c$ and $y_r$ represent the human-preferred and human-rejected completions, respectively.
The dataset is randomly split into three 4K-sample subsets: one for fine-tuning both weak and strong base models, another for weak supervision via weak model predictions, and the last for testing and evaluation. 
The dataset is randomly divided into three distinct subsets:
(1) 4K ground truth samples for fine-tuning both weak and strong base language models;
(2) A held-out set of 4K samples, where labels are predicted by the weak model and used to provide weak supervision for training the strong model;
(3) The remaining 4K samples, reserved for testing and evaluating the performance of all models.

\paragraph{Model.}
To explore weak-to-strong generalization, we utilize GPT-2 series ~\citep{radford2019language} (including GPT-2-Base, GPT-2-Medium, GPT-2-Large, and GPT-2-XL) and Pythia series~\citep{biderman2023pythia} (including Pythia-70M, Pythia-160M, Pythia-410M and Pythia-1B).
For each model, we append a linear projection head to facilitate logit predictions for each completion pair $\Tilde{x}=(x;y_c,y_r)$. 
Consequently, the task can be framed as a binary classification problem, where the model $F$ predicts the soft label as 
$$F(\Tilde{x}) = \text{Sigmoid}(F(y_c)-F(y_r)).$$

\paragraph{Training.}
The models are trained via KL divergence loss. 
Training details are in~\cref{exp_llm_training_detail}.

\paragraph{Metric.}
To evaluate whether a model $F$ can effectively distinguish between chosen and rejected completions ($y_c$ and $y_r$) for a given prompt $x$, we aim for $F$ to assign a higher score to the chosen completion compared to the rejected one.
Specifically, this requires $F(y_c)-F(y_r)>0$ for each completion pair $\Tilde{x}=(x;y_c,y_r)$, which implies $F(\Tilde{x})>0.5$.
Accordingly, the test accuracy of a model $F$ is reported as the fraction of predictions that satisfy $F(\Tilde{x})>0.5$.

\subsubsection{Results and Analysis}
The generalization and calibration performance for Pythia and GPT-2 series are shown in~\cref{exp_llm_main}.
To further investigate how the optimization result $\dist(F_w, F_{sw})$ affect generalization and calibration, we increase the number of epochs to train a strong model that more closely imitates the weak model. The corresponding results are in~\cref{exp_llm_ablation}.

\paragraph{Main results.}
\cref{fig:pythia_acc} and~\cref{fig:gpt_acc} demonstrate that, for the same strong model, the generalization of W2SG increases when supervised by a weak model of greater capacity.
This experimental result is consistent with~\cref{lemma:upper_lower_inf}.
Interestingly, in \cref{fig:pythia_calibration} and~\cref{fig:gpt_calibration}, we observe that for the same weak model as the teacher, a stronger student model with higher capacity tends to exhibit larger ECE. 
This aligns with previous findings on the negative correlation between model capacity and calibration~\citep{guo2017calibration}.

\paragraph{Ablation study.}
We extend the training epochs and design a series of teacher-student pairs with increasing model capacities. 
In particular, we employ Pythia-70M as the weak teacher to supervise Pythia-160M, 410M, and 1B. 
We also utilize GPT-2-Base as the weak teacher for supervising GPT-2-Medium, Large, and XL.
\cref{exp_llm_ablation} illustrates that as we increase the number of epochs to reduce $\dist(F_w, F_{sw})$, \textit{there is a simultaneous decline in both the accuracy and calibration error of other strong models}.
Taking the Pythia series as an example, 
\cref{fig:pythia_acc} and~\cref{fig:pythia_calibration} demonstrate that 
Pythia-70M achieves the lowest accuracy and best ECE performance among the Pythia models. 
While~\cref{theorem:calibration} indicates that reducing $\dist(F_w, F_{sw})$ causes the accuracy and calibration results of strong models to converge toward those of the weak model, 
our experiments show that increasing the number of epochs leads to reduced accuracy and ECE for Pythia-160M, 410M, and 1B.
In other words, the accuracy and ECE of strong models approach those of the weak model, consistent with~\cref{lemma:upper_lower_inf} and~\cref{theorem:calibration}. 
And this trend is also observed in the GPT-2 series.

\paragraph{Overfitting blocks W2SG.}
As the number of epochs increases, \textit{the accuracy of GPT-2-XL drops even below that of GPT-2-Base} (90.95\%). 
This is attributed to the strong expressive power of GPT-2-XL, which leads to overfitting to the weak supervision provided by GPT-2-Base.
Note that the upper bounds derived in~\cref{lemma:upper_lower_inf} and~\cref{theorem:calibration} do not guarantee that the strong model will outperform the weak model in terms of both generalization performance and calibration properties.
The underlying intuition is that if a strong model overfits to the weak supervision, it may closely mimic the weak model's generalization and calibration behavior. 
Consequently, the strong model could end up performing on par with or potentially even worse than the weak model.

\section{Results Beyond Squared Loss} \label{subsec:recover_quantify}

In regression problems under some assumptions, \citet{charikar2024quantifying} proves that the strong model’s error is smaller than the weak model’s, with the gap at least the strong model’s error on the weak labels.
This observation naturally raises the following question:
\textit{Can their proof be extended from squared loss to output distribution divergence?}
In this section, we show how to theoretically bridge the gap between squared loss and KL divergence within the overall proof framework established in~\citet{charikar2024quantifying}.
To begin with, we restate an assumption used in previous study.
\begin{assumption}[Convexity Assumption~\citep{charikar2024quantifying}] \label{convex_set}
The strong model learns fine-tuning tasks from a convex set $\cF_{s}$.
\end{assumption}

To satisfy this assumption, $\cF_s$ can be the class of all linear functions~\citep{charikar2024quantifying}.
It is validated by practice: a popular way to fine-tune a pre-trained model on task-specific data is by tuning the last linear layer of the model~\citep{howard2018universal,kumar2022fine}.

\subsection{Upper Bound (Realizability)} \label{subsub:realize}

Firstly, we consider the case where $\exists f_s \in \cF_s$ such that $F_s = F^\star$ (also called ``Realizability''~\citep{charikar2024quantifying}).
It means we can find a $f_s$ such that $f_s \circ h_s = f^\star \circ h^\star$.
This assumption implicitly indicates the strong power of pre-training. 
It requires that the representation $h_s$ has learned extremely enough information during pre-training, which is reasonable in modern large language models pre-trained on very large corpus~\citep{touvron2023llama,achiam2023gpt}.
We state our result in the realizable setting, which corresponds to Theorem 1 in~\citet{charikar2024quantifying}.

\begin{theorem}[Proved in \cref{proof_theorem_1-main}]
\label{thm:realizable-main}
Given $F^\star$, $F_w$ and $F_{sw}$ defined above.
Consider $\cF_s$ that satisfies Assumption~\ref{convex_set}. 
Consider W2SG using reverse KL divergence loss:
\begin{align*}
    f_{sw} = \argmin_{f \in \cF_{s}}\; \dist(f \circ h_s, f_w \circ h_w).
\end{align*}
Assume that $\exists f_s \in \cF_s$ such that $F_s = F^\star$.
Then
\begin{align} \label{eqn:realizable-main}
    \dist(F^\star, F_{sw}) \le \dist(F^\star, F_w) - \dist(F_{sw}, F_w).
\end{align}
\end{theorem}
\begin{remark}
    The corresponding theorem and proof in the case of forward KL divergence loss is provided in~\cref{thm:realizable} from~\cref{proof_theorem_1}, under an additional assumption.
\end{remark}

In contrast to the symmetric squared loss studied in prior work~\citep{charikar2024quantifying}, the emergence of the reverse KL divergence is inherently tied to the asymmetric properties of the KL divergence.
Although extending previous work to both forward and reverse KL divergences presents significant technical challenges, our results demonstrate the theoretical guarantees of W2SG in these settings.
In Inequality~\eqref{eqn:realizable-main}, the left-hand side represents the error of the weakly-supervised strong model on the true data. 
On the right-hand side, the first term denotes the true error of the weak model, while the second term captures the disagreement between the strong and weak models, which also serves as the minimization objective in W2SG. 
This inequality indicates that the weakly-supervised strong model improves upon the weak model by at least the magnitude of their disagreement, $\dist(F_{sw}, F_w)$.
To reduce the error of $F_{sw}$, \cref{thm:realizable-main} aligns with~\cref{lemma:upper_lower_inf}, highlighting the importance of selecting an effective weak model and the inherent limitations of the optimization objective in W2SG.


\subsection{Upper Bound (Non-Realizability)}

We relax the ``realizability'' condition and draw $n$ i.i.d. samples.
We provide the result in the ``unrealizable'' setting, where the condition $F_s = F^\star$ may not be satisfied for any $f_s \in \cF_s$.
It corresponds to Theorem 2 in~\citet{charikar2024quantifying}.

\begin{theorem}[Proved in~\cref{proof_non-realizable-main}] \label{thm:non-realizable-finite-samples-main}
Given $F^\star$, $F_w$ and $F_{sw}$ defined above.
Consider $\cF_s$ that satisfies~\cref{convex_set}.
Consider weak-to-strong generalization using reverse KL:
\begin{align*}
    & f_{sw} = \argmin_{f \in \cF_{s}}\; \dist(f \circ h_s, f_w \circ h_w),
    \\ & \hat{f}_{sw} = \argmin_{f \in \cF_{s}}\; \hat{d}_{\cP}(f \circ h_s, f_w \circ h_w),
\end{align*}
Denote $\dist(F^\star, F_s) = \eps$. 
With probability at least $1-\delta$ over the draw of $n$ i.i.d. samples, there holds
\begin{multline*} 
\dist(F^\star, \hat{F}_{sw}) \le \dist(F^\star, F_w) - \dist(\hat{F}_{sw}, F_w) \\ + \cO(\sqrt{\eps}) +  \cO\left(\sqrt{\frac{\cC_{\cF_s}}{n}}\right) + \cO\left(\sqrt{\frac{\log(1/\delta)}{n}}\right),
\end{multline*}
where $\cC_{\cF_s}$ is a constant capturing the complexity of the function class $\cF_s$, and the asymptotic notation is with respect to $\eps \to 0, n \to \infty$.
\end{theorem}

\begin{remark}
    The extension to forward KL divergence loss is provided in~\cref{thm:non-realizable-finite-samples} from~\cref{proof_non-realizable}, under an additional assumption.
\end{remark}

Compared to Inequality~\eqref{eqn:realizable-main}, this bound introduces three another error terms: $\cO(\sqrt{\eps})$ arises due to the non-realizability assumption, and diminishes as the strong ceiling model $F_s$ becomes more expressive.
The remaining two error terms arise from the strong model $\hat{F}_{sw}$ being trained on a finite weakly-labeled sample. They also asymptotically approach zero as the sample size increases.

\begin{figure*}[t]
  \centering
  \subfigure[Realizable (pre-training).]{
    \includegraphics[width=0.31\textwidth]{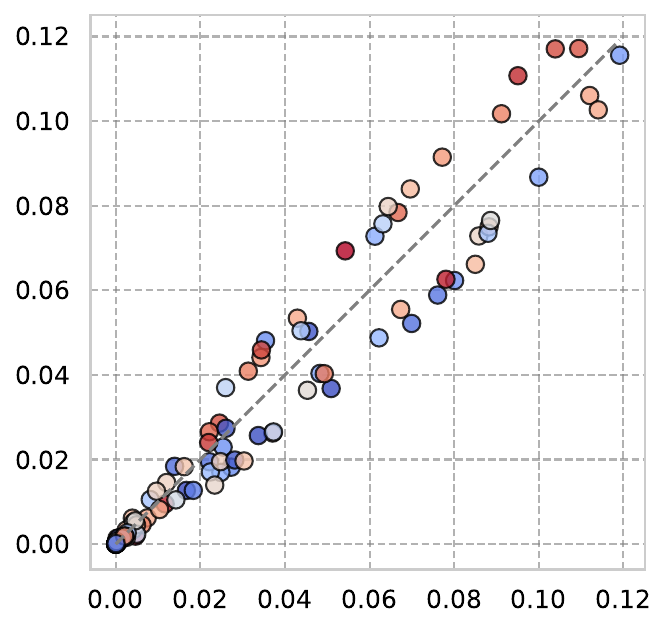}
  }
  \label{fig3:a}
  \subfigure[Non-realizable (pre-training).]{
    \includegraphics[width=0.31\textwidth]{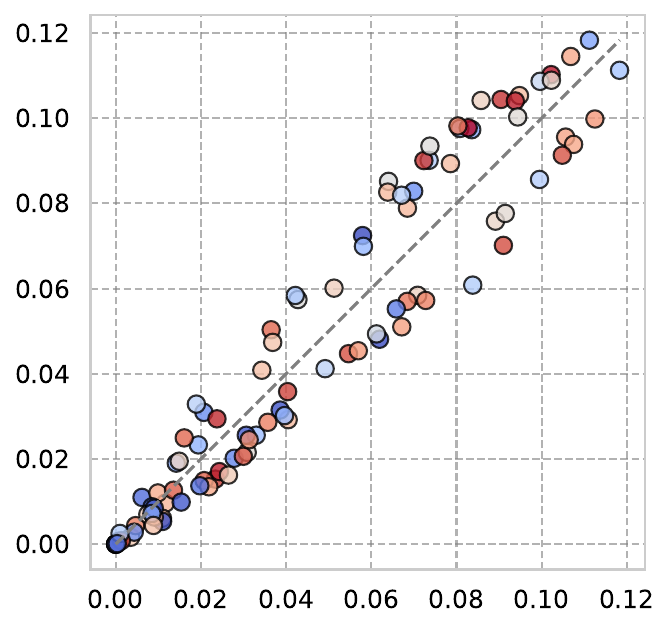}
    \label{fig3:b}
  }
  \subfigure[Non-realizable (perturbation).]{
    \includegraphics[width=0.31\textwidth]{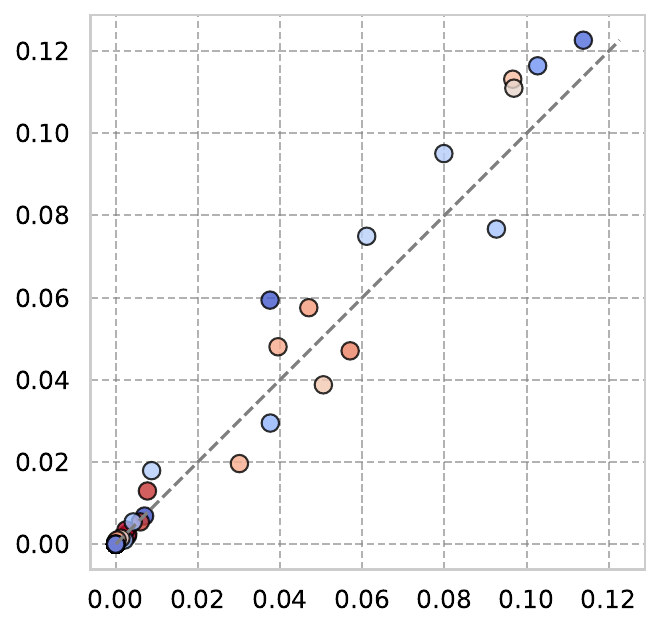}
    \label{fig3:c}
  }
  \subfigure[Realizable (pre-training).]{
    \includegraphics[width=0.31\textwidth]{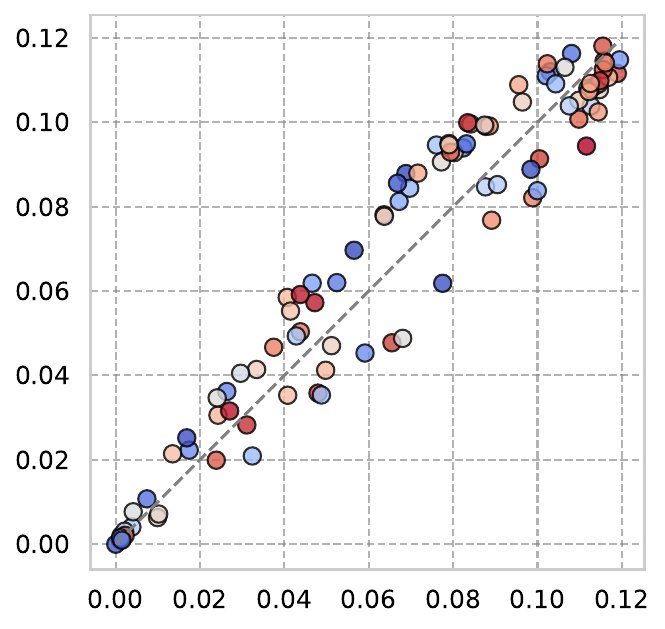}
    \label{fig3:d}
  }
  \subfigure[Non-realizable (pre-training).]{
    \includegraphics[width=0.31\textwidth]{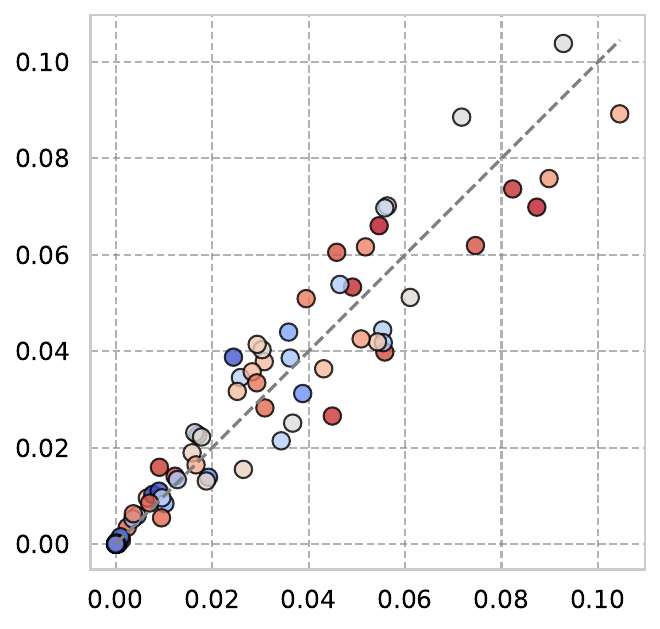}
    \label{fig3:e}
  }
  \subfigure[Non-realizable (perturbation).]{
    \includegraphics[width=0.31\textwidth]{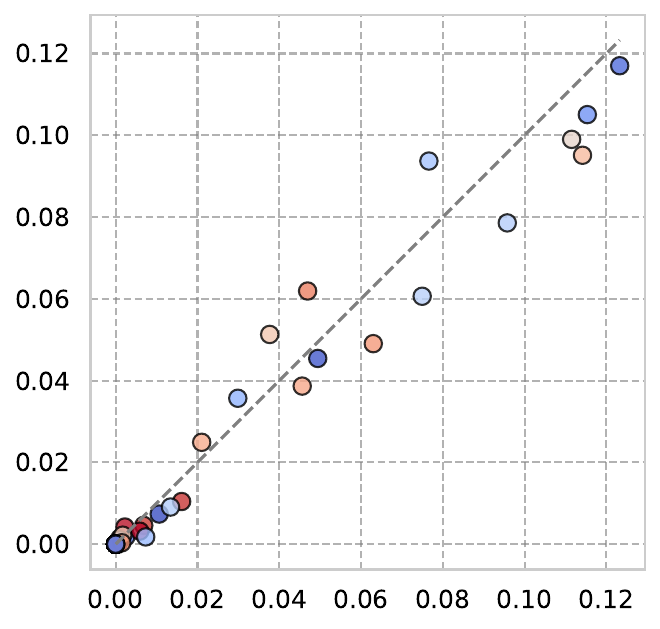}
    \label{fig3:f}
  }
  \vspace{-6pt}
  \caption{Synthetic experiments following~\citet{charikar2024quantifying} using reverse KL divergence loss (\textbf{a-c}) and forward KL divergence loss (\textbf{d-f}). 
  Each point corresponds to a task and the gray dotted line represents $y=x$. 
  $h^{\star}$ is a 16-layer MLP. 
  (\textbf{a,d}) Realizable (pre-training): $h_s=h^\star$, and $h_w$ is a 2-layer MLP obtained by pre-training. (\textbf{b,e}) Non-realizable (pre-training): $h_s$ is an 8-layer MLP, and $h_w$ is a 2-layer MLP. Both $h_s$ and $h_w$ are obtained by pre-training. (\textbf{c,f}) Non-realizable (perturbation): Both $h_s$ and $h_w$ are obtained by directly perturbing the weights in $h^{\star}$:  $h_s=h^{\star}+ \mathcal{N}\left(0,0.01\right)$, and $h_w=h^{\star}+\mathcal{N}\left(0,9\right)$.}
  \label{syn_result:reverse}
  \vspace{-11pt}
\end{figure*}

\subsection{Synthetic Experiments} \label{section:syn_exp}
In this section, we conduct experiments on synthetically generated data to validate the theoretical results in~\cref{subsec:recover_quantify}.
We extend the synthetic experiments from~\citet{charikar2024quantifying} using exactly the same setting except for the loss function defined in~\cref{def:kl_dist_emp}.
In particular, the data follows a zero-mean normal distribution. The weak, strong and ground truth representations ($h_w$, $h_s$ and $h^\star$) are implemented as randomly initialized MLPs. 
One 'Realizable' and two 'Non-realizable' settings are considered.
The full details are in~\cref{appendix:syn_train}.
To validate~\cref{thm:realizable-main}-\ref{thm:non-realizable-finite-samples-main} and visualize the trend clearly, we plot $\dist(F^\star, F_w)-\dist(F^\star, F_{sw})$ on the $x$-axis versus $\dist(F_{sw}, F_w)$ on the $y$-axis. The results are presented in~\cref{syn_result:reverse}(a)-(c).
We also examine forward KL divergence loss in~\cref{syn_result:reverse}(d)-(f). 

For the reverse KL divergence loss, the points in cluster around the line $y=x$.
This suggests that 
$\dist(F^\star, F_w)-\dist(F^\star, F_{sw}) \approx \dist(F_{sw}, F_w)$, which aligns with previous results~\citep{charikar2024quantifying} for squared loss.
It is also consistent with our theoretical results in~\cref{subsec:recover_quantify}, suggesting that the improvement over the weak teacher can be quantified by the disagreement between strong and weak models.
For the forward KL divergence loss, the observed trend closely mirrors that of reverse KL. 
The dots also are around the line $y=x$, suggesting the same relationship for the forward KL.

\section{Conclusion}
This paper provides a theoretical analysis of W2SG. 
In the classification setting, we establish upper and lower bounds for both generalization and calibration errors of the strong model, revealing that the primary limitations arise from the weak model and the optimization objective. 
These bounds emphasize two critical insights: (1) the weak model must demonstrate strong generalization and calibration performance, and (2) the strong model should avoid excessive training to prevent overfitting on weak supervision.
In the regression setting, we extend previous work to output distribution divergence loss, proving that a strong model can outperform its weak teacher by at least their disagreement under certain assumptions. 
Overall, we hope this work deepens the understanding of W2SG and inspires future research on its foundations.

\section*{Limitations}

While our work offers theoretical insights and empirical validation in W2SG, several limitations warrant discussion.
First, although the theoretical framework in~\cref{section:universal_result} provides broad conceptual understanding, the assumptions in \cref{subsec:recover_quantify} may not fully align with real-world LLM deployment scenarios. This challenge, however, is common across most theoretical analyses of W2SG. Despite this limitation, our findings establish a foundational framework for future research to refine W2SG theory in LLMs.
Second, our empirical evaluation is limited to two established alignment-focused binary classification tasks using relatively small-scale models. While the results validate our theory, further investigation is needed to assess the generalizability of our approach to more diverse datasets and larger model architectures. Addressing this in future studies will be critical for validating the broader applicability of our theory.

\section*{Broader Impact and Ethics Statement}

This work on weak-to-strong generalization aims to improve the alignment of superhuman models with human values. 
While our theoretical and empirical insights highlight the potential of this approach, we acknowledge the risks of propagating biases or errors from the weak model to the strong model. 
To address these concerns, we emphasize the importance of ensuring the weak model's generalization and calibration, as well as carefully balancing the strong model's optimization to avoid over-reliance on weak supervision. 
We encourage rigorous testing, transparency, and ongoing monitoring in real-world applications to ensure the safe and ethical deployment of such systems. 
Our work contributes to the broader effort of aligning advanced AI with human values, but its implementation must prioritize fairness, accountability, and safety.


\bibliography{custom}

\newpage
\onecolumn
\appendix

\tableofcontents
\newpage

{\LARGE \centering \textbf{Appendix} \par}
\section{Further Related Work} \label{appendix:related_work}

\paragraph{Teacher-student learning paradigm.}
The student-teacher training paradigm~\citep{meseguer2019dali,meng2019conditional}, which involves first training a teacher model and then using its outputs (e.g., pseudo-labels or soft targets) to guide the training of a student model, has become a cornerstone in various machine learning domains.
This approach is particularly prominent in knowledge distillation~\citep{hinton2015distilling,beyer2022knowledge}, semi-supervised learning~\citep{tarvainen2017mean} and domain adaptation~\citep{shu2018dirt}.
It can also be used in other fields like curriculum learning~\citep{matiisen2019teacher} and continual learning~\citep{lee2021continual}.
However, most prior work assumes that the teacher is either more capable or at least comparable to the student in terms of model capacity or performance. 
In contrast, weak-to-strong generalization~\citep{burns2023weak} explores a less studied setting where the student model is significantly more capable than the teacher, which is different from the traditional assumptions of the student-teacher framework.
By theoretically investigating this setting, we aim to uncover novel insights into the capabilities and limitations of weak-to-strong generalization.
We hope that our theoretical investigation into it will serve as a catalyst for advancements not only in the domain of super-alignment but also in the broader landscape of teacher-student learning paradigm.

\paragraph{Weakly-supervised learning.}
Weakly-supervised learning has emerged as a powerful paradigm to address the challenges of limited labeled data by leveraging weak supervision~\citep{ratner2020snorkel}.
Such weak supervision may be incomplete (i.e., only a small subset of labels are given), inexact (i.e., only coarse-grained labels are given) and inaccurate (i.e., the given labels are noisy)~\citep{zhou2018brief}.
This problem setting is also closely related to label noise~\citep{song2022learning} and semi-supervised learning~\citep{van2020survey}.
To address the problem of weakly-supervised learning, practical trials leverage these various forms of weak supervision, such as utilizing noisy labels~\citep{cheng2020weakly}, coarse-grained labels~\citep{oquab2015object}, and incomplete annotations~\citep{papadopoulos2017training}.
And most of them improving model performance within the limitations of weak supervision.
In contrast, weak-to-strong generalization explores a distinct yet related direction: it investigates how a strong model, when trained on weak supervision, can not only correct the errors of the weak supervisor but also generalize to instances where the weak supervisor is uncertain or incorrect~\citep{burns2023weak,yang2024super}. 
We hope that our theoretical exploration of weak-to-strong generalization can inspire not only the field of super-alignment but also research in weakly-supervised learning.

\paragraph{Calibration.}
Calibration is an important concept about uncertainty estimation and confidence~\citep{guo2017calibration,kuleshov2018accurate,kumar2019verified,mehrtash2020confidence} in machine learning.
There are several kinds of definition for calibration. For instance, taking expectation conditioned on the data distribution~\citep{kull2019beyond,kumar2019verified,roelofs2022mitigating} (Also, see Definition 3 in~\citep{pleiss2017fairness} and Equation (2) in~\citep{liu2019implicit} in the fairness literature), and
(2) taking expectation conditioned on the probability score~\citep{naeini2015obtaining,guo2017calibration}.
Researchers also investigate the calibration in natural language processing~\citep{desai2020calibration,guo2021overview,ulmer2022exploring,chen-etal-2023-close}.
In recent years, the calibration of large language models has garnered significant attention~\citep{zhu2023calibration,tian2023just,liang2022holistic}, with a thorough survey provided in~\citep{geng2024survey}.
However, to the best of our knowledge, while confidence issues in weak-to-strong generalization have been investigated in~\citep{burns2023weak}, the role of calibration has not been sufficiently investigated.
In this paper, we theoretically demonstrate how strong model's calibration is affected in W2SG.
And we believe that calibration warrants further in-depth investigation in this field.

\paragraph{Information-theoretic analysis.}
Information-theoretic analysis is commonly employed to bound the expected generalization error in supervised learning~\citep{russo2016controlling,xu2017information}, with subsequent studies providing sharper bounds~\citep{bu2020tightening,wang2023tighter}. 
These bounds have been used to characterize the generalization ability of stochastic gradient-based optimization algorithms~\citep{pensia2018generalization}. 
Furthermore, this theoretical framework has been extended to diverse settings, including meta-learning~\citep{chen2021generalization}, semi-supervised learning~\citep{aminian2022information}, transductive learning~\citep{tang2023information}, and domain adaptation~\citep{wang2022information}. 
For a comprehensive overview of these developments, we refer readers to the recent monograph by~\citet{hellstrom2023generalization}. 
Nonetheless, despite its extensive application across various domains, the information-theoretic analysis of super-alignment~\citep{openai_superalignment}, particularly in the context of weak-to-strong generalization, remains largely underexplored.
In this paper, we use KL divergence to analyze weak-to-strong generalization, which is not considered in previous work. 
KL divergence is an information-theoretic measure between two probability distributions in information theory~\citep{cover1999elements}.
And how to extend it to other information-theoretic measures remains an open question and warrants further exploration in future work.

\section{Main Proof}

\subsection{Proof of~\cref{lemma:upper_lower_inf}} \label{proof_lemma_inf}

We first state some preliminaries for the proof.

\begin{lemma}[Donsker and Varadhan’s variational formula~\citep{donsker1983asymptotic}] \label{lemma:donsker}
Let $Q, P$ be probability measures on $\cX$, for any bounded measurable function $f: \cX \rightarrow \mathbb{R}$, we have $$\mathrm{D}_{\mathrm{KL}}(Q \| P)=\sup _f \mathbb{E}_{x \sim Q}[f(x)]-\log \mathbb{E}_{x \sim P}[\exp f(x)].$$
\end{lemma}

\begin{lemma}[Hoeffding's lemma] \label{hoeffding_lemma}
Let $X \in \R$ such that $a \leq X \leq b$. Then, for all $\lambda \in \mathbb{R}$,
$$\mathbb{E}\left[e^{\lambda(X-\mathbb{E}[X])}\right] \leq \exp \left(\frac{\lambda^2(b-a)^2}{8}\right).$$
\end{lemma}

\begin{definition}[Subgaussian random variable]
A random variable $X \in \R$ is $\sigma$-subgaussian if for any $\rho$, 
$$\log \mathbb{E} \exp (\rho(X-\mathbb{E} X)) \leq \rho^2 \sigma^2 / 2.$$
\end{definition}

\paragraph{Notation of probability distribution for the model output.}
We define the corresponding probability distributions for prediction of $F_{sw}$ and $F_w$.
Recall that for $F_w, F_{sw}: \cX \to \cY$ and $x \in \cX$:
$$\dist(F_w, F_{sw}) = \bE_x \left[ \sum_{j=1}^k [F_w(x)]_j \log \frac{[F_w(x)]_j}{[F_{sw}(x)]_j} \right] = \bE_x \left[ \mathrm{D}_{\mathrm{KL}}(F_w(x), F_{sw}(x)) \right],$$
where $\mathrm{D}_{\mathrm{KL}}$ is the discrete version of KL divergence. 
$\forall x \in \cX$, we know that $\sum_{j=1}^k [F_w(x)]_j = 1$. Therefore, given the class space $C_k = \{ 1, \cdots, k \}$, we define a probability distribution $\cP_{w}(x)$ with the probability density function $p_w$, where $j \in C_k$ and 
\begin{align} \label{def_new_distribution}
    p_w(j)=[F_w(x)]_j.
\end{align}
Using this method, we also define the probability distribution $\cP_{sw}(x)$ for $F_{w}(x)$.

Now we start the proof.

\begin{proof}

For better readability, we divide the proof into several steps.

\paragraph{The first step.}
Given the probability distributions $\cP_{w}(x)$ and $\cP_{sw}(x)$ above,
the first step is motivated by Lemma A.2 from~\citep{wang2022information}.
For any $x \in \cX$, $j \in C_k$, $g: C_k \to \R$ and assume that $g$ is $\sigma$-subgaussian (we will specify $\sigma$ later).
Let $f=t \cdot g$ for any $t \in \R$.
We have
\begin{align*}
& \mathrm{D}_{\mathrm{KL}}\left( F_w(x) \| F_{sw}(x) \right) \\ = & \mathrm{D}_{\mathrm{KL}}(\cP_w(x) \| \cP_{sw}(x)) \\
= & \sup_t \mathbb{E}_{j^{\prime} \sim \cP_w(x)}\left[t \cdot g(j') \right]-\log \mathbb{E}_{j \sim \cP_{sw}(x)}[\exp \left(t \cdot g(j) \right)] \tag{\cref{lemma:donsker}} \\
= & \sup_t \mathbb{E}_{j^{\prime} \sim \cP_w(x)}\left[tg\left(j^{\prime}\right)\right]-\log \mathbb{E}_{j \sim \cP_{sw}(x)}\left[\exp t\left(g(j)-\mathbb{E}_{j \sim \cP_{sw}(x)}[g(j)]+\mathbb{E}_{j \sim \cP_{sw}(x)}[g(j)]\right)\right] \\
= & \sup_t \mathbb{E}_{j^{\prime} \sim \cP_w(x)}\left[t g\left(j^{\prime}\right)\right]-\mathbb{E}_{j \sim \cP_{sw}(x)}[tg(j)]-\log \mathbb{E}_{j \sim \cP_{sw}(x)}\left[\exp t\left(g(j)-\mathbb{E}_{j \sim \cP_{sw}(x)}[g(j)]\right)\right] \\
\geq & \sup _t t\left(\mathbb{E}_{j^{\prime} \sim \cP_w(x)}\left[g\left(j^{\prime}\right)\right]-\mathbb{E}_{j \sim \cP_{sw}(x)}[g(j)]\right)-t^2 \sigma^2 / 2. \tag{Subgaussianity}
\end{align*}

\paragraph{The second step.}
The second step is associating the above result with $\dist(F_w, F_{sw})$.
In particular, by taking expectations of $x$ on both sides of the above inequality, we obtain
$$\dist(F_w, F_{sw}) = \bE_x \mathrm{D}_{\mathrm{KL}}\left( F_w(x) \| F_{sw}(x) \right) \ge \sup _t \underbrace{t\left(\bE_x\mathbb{E}_{j^{\prime} \sim \cP_w(x)}\left[g\left(j^{\prime}\right)\right]-\bE_x\mathbb{E}_{j \sim \cP_{sw}(x)}[g(j)]\right)-t^2 \sigma^2 / 2}_{\phi(t)}.$$

Note that $\phi(t)$ is a quadratic function of $t$.
Therefore, by AM–GM inequality, we find the maximum of this quadratic function:
\begin{align*}
    \phi(t) \le \frac{1}{2\sigma^2}\left(\bE_x\mathbb{E}_{j^{\prime} \sim \cP_w(x)}\left[g\left(j^{\prime}\right)\right]-\bE_x\mathbb{E}_{j \sim \cP_{sw}(x)}[g(j)]\right)^2 = \sup _t \phi(t) \le \dist(F_w, F_{sw}).
\end{align*}

Subsequently, there holds
\begin{align} \label{ineq:lower_upper_kl_loss}
\left|\bE_x\mathbb{E}_{j^{\prime} \sim \cP_w(x)}\left[g\left(j^{\prime}\right)\right]-\bE_x\mathbb{E}_{j \sim \cP_{sw}(x)}[g(j)]\right| \le \sqrt{2\sigma^2 \dist(F_w, F_{sw})}.
\end{align}

\paragraph{The third step.}
The third step is constructing $g$ to associate the above result with $\dist(F^\star, F_{sw})$ and $\dist\left( F^\star, F_w  \right)$.
Specifically, given a probability distribution $\cP_g$ with the density function $p_g$, we define function $g: C_k \to (0,1]$ associated with $\cP_g$: 
$$g(j) \triangleq \frac{[F^\star(x)]_j}{p_g(j)} \log \frac{[F^\star(x)]_j}{p_g(j)}, \quad \text{for} \ j \in C_k.$$

We have
\begin{align*}
    \bE_x\mathbb{E}_{j \sim \cP_g} \left[g(j)\right] & = \bE_x \bE_{j \sim \cP_g} \left[\frac{[F^\star(x)]_j}{p_g(j)} \log \frac{[F^\star(x)]_j}{p_g(j)} \right] \\
    & = \bE_x \left[\sum_{j \in C_k} p_g(j) \cdot \frac{[F^\star(x)]_j}{p_g(j)} \cdot \log \frac{[F^\star(x)]_j}{p_g(j)} \right] \\
    & = \bE_x \left[\sum_{j \in C_k} [F^\star(x)]_j \cdot \log \frac{[F^\star(x)]_j}{p_g(j)} \right]
\end{align*}

Recall the definition of $\cP_{sw}$ and $\cP_w$ in~\eqref{def_new_distribution}, we replace $\cP_g$ with $\cP_{sw}$ and $\cP_w$ in the above equation:
\begin{align*}
    & \bE_x\mathbb{E}_{j^{\prime} \sim \cP_{sw}}\left[g\left(j^{\prime}\right)\right] = \bE_x \left[ \sum_{j=1} [F^\star(x)]_j \log \frac{[F^\star(x)]_j}{[F_{sw}(x)]_j} \right] = \dist(F^\star, F_{sw}), \\
    & \bE_x\mathbb{E}_{j \sim \cP_w}[g(j)] = \bE_x \left[ \sum_{j=1} [F^\star(x)]_j \log \frac{[F^\star(x)]_j}{[F_{w}(x)]_j} \right] = \dist(F^\star, F_{w}).
\end{align*}

Substitute the above into~\eqref{ineq:lower_upper_kl_loss}:
\begin{align} \label{ineq:temp-1}
    \left| \dist(F^\star, F_{sw})-\dist(F^\star, F_{w}) \right| \le \sqrt{2\sigma^2 \dist(F_w, F_{sw})}.
\end{align}

\paragraph{The final step.}
Finally, we obtain the subgaussian factor $R$ of function $g$ by using the fact that $g$ is bounded.
For simplicity, we use Hoeffding's Lemma (\cref{hoeffding_lemma}) to obtain the subgaussian factor $R$. However, it can be more precisely determined using advanced techniques in learning theory literature (for instance, see Remark 2.14 in~\citep{lialgorithmic}, where $\alpha=2$ recovers the subgaussian setting).

Recall that the output domain $\cY \subseteq \R^k$, where $\forall y = (y_1, \cdots, y_k)^T \in \cY$, there holds $\sum_{i=1}^k y_i=1$ and $0 < y_i \le 1$.
In other words, $\exists \gamma>0$ such that $0 < \gamma \le y_i \le 1$.
It means that $g(j) \in [-\frac{1}{\gamma} \log \frac{1}{\gamma}, \frac{1}{\gamma} \log \frac{1}{\gamma}]$.
According to~\cref{hoeffding_lemma}, $\forall \lambda \in \R$, we have
$$\mathbb{E}\left[e^{\lambda(g(j)-\mathbb{E}[g(j)])}\right] \leq \exp \left(\frac{\lambda^2 \left(\frac{1}{\gamma} \log \frac{1}{\gamma} \right)^2}{2}\right).$$
In other words, $g(j)$ is $\sigma$-subgaussian, where $\sigma=\frac{1}{\gamma} \log \frac{1}{\gamma}$.
Substitute it into~\cref{ineq:temp-1} and we obtain:
\begin{align*}
    \left| \dist(F^\star, F_{sw}) - \dist\left( F^\star, F_w  \right) \right| \le C_1 \sqrt{\dist(F_w, F_{sw})},
\end{align*}
where the constant $C_1 = \frac{\sqrt{2}}{\gamma} \log \frac{1}{\gamma}$. The proof is complete.
\end{proof}

\subsection{Extension of~\cref{lemma:upper_lower_inf}} \label{proof:lower_upper}

In this section, we extend~\cref{lemma:upper_lower_inf} to output distribution divergence in regression.

\begin{proof}

Denote $ \kl(f \| g) = \mathrm{D}_{\mathrm{KL}}(\cP_f \| \cP_g) =\int_{\cX} f(x) \log \frac{f(x)}{g(x)} d x$.
Let $(\cX, \mathcal{F}, \cP_{sw}),(\cX, \mathcal{F}, \cP_w)$ be two probability spaces.
Denoting $F_{sw}$ and $F_w$ the densities of the measures.
Therefore,
$$\int_{\cX} F_{sw}(x) d x = \int_{\cX} F_w(x) d x = 1.$$

Let $x \in \cX$, $g: \cX \to \R$ and assume that $g$ is $R$-subgaussian (we will specify $R$ later).
Let $f=t \cdot g$ for any $t \in \R$. By~\cref{lemma:donsker}, we have
\begin{align*}
\kl\left( F_w \| F_{sw} \right) & = \mathrm{D}_{\mathrm{KL}}(\cP_w \| \cP_{sw}) \\ & = \sup_t \mathbb{E}_{x^{\prime} \sim \cP_w}\left[t \cdot g\left(x^{\prime}\right)\right]-\log \mathbb{E}_{x \sim \cP_{sw}}[\exp \left(t \cdot g(x) \right)] \\
& =\sup _t \mathbb{E}_{x^{\prime} \sim \cP_w}\left[tg\left(x^{\prime}\right)\right]-\log \mathbb{E}_{x \sim \cP_{sw}}\left[\exp t\left(g(x)-\mathbb{E}_{x \sim \cP_{sw}}[g(x)]+\mathbb{E}_{x \sim \cP_{sw}}[g(x)]\right)\right] \\
& =\sup _t \mathbb{E}_{x^{\prime} \sim \cP_w}\left[t g\left(x^{\prime}\right)\right]-\mathbb{E}_{x \sim \cP_{sw}}[tg(x)]-\log \mathbb{E}_{x \sim \cP_{sw}}\left[\exp t\left(g(x)-\mathbb{E}_{x \sim \cP_{sw}}[g(x)]\right)\right] \\
& \geq \sup _t \underbrace{t\left(\mathbb{E}_{x^{\prime} \sim \cP_w}\left[g\left(x^{\prime}\right)\right]-\mathbb{E}_{x \sim \cP_{sw}}[g(x)]\right)-t^2 R^2 / 2}_{\phi(t)}. \tag{Subgaussianity}
\end{align*}

Let
$$\phi(t)=t\left(\mathbb{E}_{x^{\prime} \sim \cP_w}\left[g\left(x^{\prime}\right)\right]-\mathbb{E}_{x \sim \cP_{sw}}[g(x)]\right)-t^2 R^2 / 2,$$
which is a quadratic function of $t$.
Therefore, 
\begin{align*}
    \phi(t) \le \frac{1}{2R^2}\left(\mathbb{E}_{x^{\prime} \sim \cP_{sw}}\left[g\left(x^{\prime}\right)\right]-\mathbb{E}_{x \sim \cP_w}[g(x)]\right)^2 = \sup _t \phi(t) \le \mathrm{D}_{\mathrm{KL}}(\cP_w \| \cP_{sw}) = \kl\left( F_w \| F_{sw} \right).
\end{align*}

So we have
\begin{align} \label{ineq:kl_upp}
\left|\mathbb{E}_{x^{\prime} \sim \cP_{sw}}\left[g\left(x^{\prime}\right)\right]-\mathbb{E}_{x \sim \cP_w}[g(x)]\right| \le \sqrt{2R^2 \kl\left( F_w \| F_{sw} \right)}.
\end{align}

Given a probability space $(\cX, \mathcal{F}, \cP_g)$ with the density function $p_g$. Define 
$$g(x) \triangleq \frac{F^\star(x)}{p_g(x)} \log \frac{F^\star(x)}{p_g(x)}.$$
So there holds
$$\mathbb{E}_{x \sim \cP_g} \left[g(x)\right] = \int_{\cX} F^\star(x) \cdot \log \frac{F^\star(x)}{p_g(x)} d x.$$

Replace $\cP_g$ with $\cP_{sw}$ and $\cP_w$:
\begin{align*}
    & \mathbb{E}_{x^{\prime} \sim \cP_{sw}}\left[g\left(x^{\prime}\right)\right] = \int_{\cX} F^\star(x) \log \frac{F^\star(x)}{F_{sw}(x)} d x = \dist(F^\star, F_{sw}), \\
    & \mathbb{E}_{x \sim \cP_w}[g(x)] = \int_{\cX} F^\star(x) \log \frac{F^\star(x)}{F_{w}(x)} d x = \dist(F^\star, F_{w}).
\end{align*}

Substitute them back into~\eqref{ineq:kl_upp} and we obtain:
\begin{align} \label{info_theory:ziqiao_result}
& \left| \dist(F^\star, F_{sw}) - \dist(F^\star, F_{w}) \right| \le \sqrt{2 R^2 \kl\left( F_w \| F_{sw} \right)}.
\end{align}

Finally, 
recall that the output domain $\cY = \{ y \in \R| 0 < y \le 1 \}$. In other words, $\exists \gamma>0$ such that $\cY = \{ y \in \R| 0 < \gamma \le y \le 1 \}$.
It means that $g(x) \in [-\frac{1}{\gamma} \log \frac{1}{\gamma}, \frac{1}{\gamma} \log \frac{1}{\gamma}]$.
According to~\cref{hoeffding_lemma}, $\forall \lambda \in \R$, we have
$$\mathbb{E}\left[e^{\lambda(g(x)-\mathbb{E}[g(x)])}\right] \leq \exp \left(\frac{\lambda^2 \left(\frac{1}{\gamma} \log \frac{1}{\gamma}\right)^2}{2}\right).$$
In other words, $g(x)$ is $R$-subgaussian, where $R=\frac{1}{\gamma} \log \frac{1}{\gamma}$.
Substitute it into~\cref{info_theory:ziqiao_result} and we obtain:
\begin{align*}
    \left| \dist(F^\star, F_{sw}) - \dist\left( F^\star, F_w \right) \right| \le \sqrt{C_1 \dist(F_w, F_{sw})},
\end{align*}
where the constant $C_1 = 2\left(\frac{1}{\gamma} \log \frac{1}{\gamma} \right)^2$.
\end{proof}

\subsection{Proof of~\cref{theorem:calibration}} \label{proof:calibration}

Total variation distance is introduced for our proof.
\begin{definition}[Total Variation Distance] \label{def:tv_distance}
Given two probability distributions $P$ and $Q$, the Total Variation (TV) distance between $P$ and $Q$ is
$$\tv(P \| Q)= \frac{1}{2} \int_{x \in \mathcal{X}} \left| P(x)-Q(x) \right| d x.$$
\end{definition}
Note that $\tv(P \| Q)\in[0,1]$. Also, $\tv(P \| Q)=0$ if and only if $P$ and $Q$ coincides, and $\tv(P \| Q)=1$ if and only if $P$ and $Q$ are disjoint.

Let the calibrated bayes score function $F^b: \cX \to \cY$ that satisfies $\forall i \in \{1, \cdots, k\}$, $[F^b(x)]_i=\prob(Y_i=1|X=x)$, where $Y=[Y_1, \cdots, Y_k]^T \in \{0,1\}^k$ and $\|Y\|_1=1$.
Now we start our proof.


\begin{proof}

Consider the definition of MCE in~\cref{def:cal_err}.
Notice that
\begin{align*}
    \textit{MCE}(F) & = \sum_{i=1}^k \expect_X \left| [F(X)]_i-\prob[Y_i=1|[F(X)]_i] \right| \\ & = \expect_X \left[ \sum_{i=1}^k \left| [F(X)]_i-\prob[Y_i=1|[F(X)]_i] \right| \right] \\
    & = \expect_X \left\| F(X)-F^b(X) \right\|_1,
\end{align*}
and $\textit{MCE}(F) \in [0,2]$.

So there holds
\begin{align*}
    \textit{MCE}(F_w) - \textit{MCE}(F_{sw}) & = \expect_X \left\| F_w(X)-F^b(X) \right\|_1 - \expect_X \left\| F_{sw}(X)-F^b(X) \right\|_1 \\
    & \le \expect_X \left\| F_w(X)-F_{sw}(X) \right\|_1 \tag{Triangle inequality} \\
    & = 2 \cdot \expect_X \tv(F_w(X), F_{sw}(X)) \\
    & \le 2 \cdot \expect_X \sqrt{1-\exp{\left(-\mathrm{D}_{\mathrm{KL}}(F_w(X),F_{sw}(X))\right)}} \tag{Bretagnolle–Huber inequality} \\
    & \le 2 \cdot \sqrt{1-\exp{\left(-\expect_X \mathrm{D}_{\mathrm{KL}}(F_w(X),F_{sw}(X))\right)}} \tag{Jensen’s inequality} \\
    & = 2 \cdot \sqrt{1-\exp{\left(-\dist(F_w,F_{sw})\right)}}. \tag{Definition of $\dist$ for KL divergence loss}
\end{align*}

Likewise,
\begin{align*}
    \textit{MCE}(F_{sw}) - \textit{MCE}(F_w) & = \expect_X \left\| F_{sw}(X)-F^b(X) \right\|_1 - \expect_X \left\| F_w(X)-F^b(X) \right\|_1 \\
    & \le \expect_X \left\| F_{sw}(X)-F_w(X) \right\|_1 \tag{Triangle inequality} \\
    & \le \expect_X \left\| F_w(X)-F_{sw}(X) \right\|_1 \tag{Symmetry} \\
    & = 2 \cdot \expect_X \tv(F_w(X), F_{sw}(X)) \\
    & \le 2 \cdot \sqrt{1-\exp{\left(-\dist(F_w,F_{sw})\right)}}. \tag{Using the derivation above}
\end{align*}

Combining the above, we have
\begin{align*}
    & \textit{MCE}(F_{sw}) - \textit{MCE}(F_w) \le 2 \cdot \sqrt{1-\exp{\left(-\dist(F_w,F_{sw})\right)}}, \\
    & \textit{MCE}(F_w) - \textit{MCE}(F_{sw}) \le 2 \cdot \sqrt{1-\exp{\left(-\dist(F_w,F_{sw})\right)}}.
\end{align*}

The proof is complete.

\end{proof}

\subsection{Proof of~\cref{thm:realizable-main}} \label{proof_theorem_1-main}

We first restate a lemma for our proof.
Recall that the strong model learns from a linear function class $\cF:\R^{d_s} \to \R$ of fine-tuning tasks. Recall also that we denote the strong model representation map by $h_s:\R^d \to \R^{d_s}$. Let $V_s = \{f \circ h_s: f \in \cF\}$ be the set of all tasks in $\cF$ composed with the strong model representation. We first observe that $V_s$ is also a convex set.
\begin{lemma}[\citet{charikar2024quantifying}]
    \label{claim:Vs-convex}
    $V_s$ is a convex set.
\end{lemma}
\begin{proof}
    Fix $f, g \in \cF$, and consider $f \circ h_s, g \circ h_s \in V_s$. Fix any $\lambda \in [0,1]$. Since $\cF$ is the linear function class so that it is a convex set, there exists $p \in \cF$ such that for all $y \in \R^{d_s}$, $p(y) = \lambda f(y) + (1-\lambda)g(y)$. Now, fix any $x \in \R^d$. Then, we have that
    \begin{align*}
        \lambda (f \circ h_s)(x) + (1-\lambda)(g \circ h_s)(x) &= \lambda f(h_s(x)) + (1-\lambda)g(h_s(x))
        = p(h_s(x)) = (p \circ h_s)(x),
    \end{align*}
    and hence $\lambda (f \circ h_s) + (1-\lambda)(g \circ h_s) = p \circ h_s \in V_s$.
\end{proof}

Now we start the proof.

\begin{proof}

For any $f,g \in \cX \to \cY$, denote $ \kl(f \| g) = \int_{\cX} f(x) \log \frac{f(x)}{g(x)} d x$.

Given any $g \in V_s$, observe that
\begin{align} \label{eqn:square-expansion-main}
    \kl(g \| F_w) &= \int_{\cX} g(x) \log \frac{g(x)}{F_w(x)} d x \nonumber \\
    &= \int_{\cX} g(x) \log \left( \frac{g(x)}{F_{sw}(x)} \cdot \frac{F_{sw}(x)}{F_w(x)} \right) d x \nonumber \\
    &=\kl(g \| F_{sw}) + \kl(F_{sw} \| F_w) - \kl(F_{sw} \| F_w) + \int_{\cX} g(x) \log  \frac{F_{sw}(x)}{F_w(x)} d x \nonumber \\
    &=\kl(g \| F_{sw}) + \kl(F_{sw} \| F_w) + \underbrace{\int_{\cX} \left( g(x) - F_{sw}(x) \right) \log  \frac{F_{sw}(x)}{F_w(x)} d x}_{Q_1}.
\end{align}
Now our goal is to judge whether $Q_1 \ge 0$.

Recall that $$f_{sw} = \argmin_{f}\; \kl(f \circ h_s \| F_w).$$
In other words, $F_{sw}$ is the \textit{projection} of $F_w$ onto the convex set $V_s$. 
Therefore:
$$\kl(g \| F_w) \ge \kl(F_{sw} \| F_w).$$
Therefore, 
\begin{align} \label{eqn:cross-term-inequality-1-main}
\kl(g \| F_{sw}) + Q_1 \ge 0. 
\end{align}

Now, fix $t \in (0,1)$, and consider the function $g=F_{sw}+t(F^\star-F_{sw})$.

\begin{align*}
    \kl(g \| F_{sw}) & = \int_\cX g(x) \log \frac{g(x)}{F_{sw}(x)}
    \\ & = -\int_\cX g(x) \log \frac{F_{sw}(x)}{g(x)} d x
    \\ & = - \int_\cX g(x) \log \left[1-\frac{t(F^\star(x)-F_{sw}(x))}{g(x)}\right] d x
    \\ & = \int_\cX g(x) \left[\frac{t(F^\star(x)-F_{sw}(x))}{g(x)} + \cO(t^2)\right] d x
    \\ & = \int_\cX \left[t(F^\star(x)-F_{sw}(x)) + \cO(t^2)\right] d x \tag{Taylor expansion}
    \\ & = \cO(t^2). \tag{$\int_\cX F_{sw}(x) d x = \int_\cX F^\star(x) d x = 1$}
\end{align*}

While
\begin{align*}
    Q_1 & = \int_{\cX} \left( g(x) - F_{sw}(x) \right) \log  \frac{F_{sw}(x)}{F_w(x)} d x
    \\ & = t \cdot \int_{\cX} \left( F^\star(x) - F_{sw}(x) \right) \log  \frac{F_{sw}(x)}{F_w(x)} d x
    \\ & = \cO(t).
\end{align*}

Recall~\cref{eqn:cross-term-inequality-1-main} that
$$\underbrace{\kl(F_{sw} \| g)}_{\cO(t^2)} + \underbrace{Q_1}_{\cO(t)} \ge 0,$$
which means $Q_1 \ge 0$.
So we have 
$$\int_{\cX} \left( F^\star(x) - F_{sw}(x) \right) \log  \frac{F_{sw}(x)}{F_w(x)} d x \ge 0.$$
Let $g=F^\star$ in~\cref{eqn:square-expansion-main} and we can prove the result $\dist(F^\star, F_{sw}) \le \dist(F^\star, F_w) - \dist(F_{sw}, F_w)$.

\end{proof}

\subsection{Extension of~\cref{thm:realizable-main}} \label{proof_theorem_1}

We first introduce some definitions.

\begin{definition}[Itakura–Saito Divergence~\citep{itakura1968analysis,fevotte2009nonnegative,prasetyawan2020sensory}]
Given two probability distributions $P$ and $Q$, the Itakura–Saito divergence between them is defined as
$$\is(P \| Q) = \int_{\cX} \left( \frac{P(x)}{Q(x)} - \log \frac{P(x)}{Q(x)} - 1 \right) d x.$$
\end{definition}

Similar to the KL divergence, the Itakura–Saito divergence is also a Bregman divergence~\citep{inderjit_bregman}.

\begin{definition}[Weighted Itakura–Saito Divergence~\citep{chu1982frequency}]
Given two probability distributions $P$ and $Q$, the weighted Itakura–Saito divergence between them is defined as
$$\wis(P \| Q) = \int_{\cX} w(x) \left( \frac{P(x)}{Q(x)} - \log \frac{P(x)}{Q(x)} - 1 \right) d x,$$
where $w(x)$ is the weight function.
\end{definition}

We also define the inner product of functions 
$$\left \langle f,g \right \rangle \triangleq \int_{\cX} f(x)g(x) d x.$$

Now we present the theoretical extension of~\cref{thm:realizable-main} to forward KL divergence loss in W2SG.

\begin{corollary} \label{thm:realizable}
Given $F^\star$, $F_w$ and $F_{sw}$ defined above.
Let $\dist$ be the output distribution divergence and consider W2SG in~\cref{eqn:fsw-population-minimizer} using forward KL divergence loss.
Assume that $\exists f_s \in \cF_s$ such that $F_s = F^\star$.
Consider $\cF_s$ that satisfies Assumption~\ref{convex_set}. 
If the weighted Itakura–Saito divergence $\wis(F^\star \| F_{sw}) \le 0$ with the weight function $w=F_{sw} - F_w$, then:
\begin{align} \label{eqn:realizable}
    d_{\mathcal{P}}\left(F_{sw}, F^\star\right) \leq d_{\mathcal{P}}\left(F_w, F^\star\right)-d_{\mathcal{P}}\left(F_w, F_{sw}\right).
\end{align}
\end{corollary}

\cref{thm:realizable} shows that~\citet{charikar2024quantifying} can be extended to our setting if we introduce another assumption, which comes from the technical challenges of theoretically analyzing the non-linear nature of KL divergence.
Denote $F^+ = \frac{F^\star}{F_{sw}} - 1 - \log{\frac{F^\star}{F_{sw}}}$, then the assumption $\wis(F^\star \| F_{sw}) \le 0$ is equivalent to $\left \langle F_{sw} - F_w, F^+ \right \rangle \le 0$.
Note that $\forall x \in \cX$, $F_+(x) \ge 0$ always holds, and a very small or large value of $\frac{F^\star(x)}{F_{sw}(x)}$ generally contributes to a large $F_+(x)$.
To make $\left \langle F_{sw} - F_w, F^+ \right \rangle \le 0$ more likely to hold, we expect $F_{sw}(x) \le F_w(x)$ if $F^\star(x)$ is small or large.
In general, since this condition cannot be guaranteed to hold universally, the inequality in~\cref{eqn:realizable-main} may fail to hold. 
This reveals a key discrepancy between the square function (as considered in \citet{charikar2024quantifying}) and the KL divergence (in this work) within the W2SG framework—a phenomenon that will be empirically validated through our experiments.

\paragraph{Proof sketch of~\citet{charikar2024quantifying}.}
For the proof technique,~\citet{charikar2024quantifying} constructs a function within a convex set.
By exploiting the property of projection and square function, they demonstrate that $\cO(t)+\cO(t^2)$ is non-negative as $t \to 0^+$. Consequently, the first-order term must be non-negative, which proves the result.

\paragraph{Proof sketch of ours.}
Extending the proof framework from Theorem 1 in~\citep{charikar2024quantifying} presents several challenges. 
First, due to the properties of KL divergence, the constructed function does not lie within the convex set. To address this issue, we employ a first-order Taylor expansion and introduce a remainder term.
Secondly, because of the remainder, we derive that $\cO(t)+\cO(t)+\cO(t^2)$ is non-negative. Consequently, we must assume that one of the first-order terms is non-positive to ensure that the other first-order term is non-negative, which allows us to prove the result.
However, if the first-order term is positive, the second first-order term might also remain non-negative.

Now we start our proof of~\cref{thm:realizable}.
Some Taylor expansion claims used in the proof (\cref{lemma:taylor}, \cref{lemma:taylor_log} and~\cref{lemma:taylor_frac}) are provided at the end of the proof.

\begin{proof}

For any $f,g \in \cX \to \cY$, denote $ \kl(f \| g) = \int_{\cX} f(x) \log \frac{f(x)}{g(x)} d x$.

Given any $g \in V_s$, observe that
\begin{align} \label{eqn:square-expansion}
    \kl(F_w \| g) &= \int_{\cX} F_w(x) \log \frac{F_w(x)}{g(x)} d x \nonumber \\
    &= \int_{\cX} F_w(x) \log \left( \frac{F_w(x)}{F_{sw}(x)} \cdot \frac{F_{sw}(x)}{g(x)} \right) d x \nonumber \\
    &= \int_{\cX} F_w(x) \log  \frac{F_w(x)}{F_{sw}(x)} d x + \int_{\cX} F_w(x) \log  \frac{F_{sw}(x)}{g(x)} d x \nonumber \\
    &=\kl(F_w \| F_{sw}) + \kl(F_{sw} \| g) - \kl(F_{sw} \| g) + \int_{\cX} F_w(x) \log  \frac{F_{sw}(x)}{g(x)} d x \nonumber \\
    &=\kl(F_w \| F_{sw}) + \kl(F_{sw} \| g) + \underbrace{\int_{\cX} \left( F_w(x) - F_{sw}(x) \right) \log  \frac{F_{sw}(x)}{g(x)} d x}_{Q_1}.
\end{align}
Now our goal is to judge whether $Q_1 \ge 0$.

Recall that $$f_{sw} = \argmin_{f}\; \kl(F_w \| f \circ h_s).$$
In other words, $F_{sw}$ is the \textit{projection} of $F_w$ onto the convex set $V_s$. 
Therefore:
$$\kl(F_w \| g) \ge \kl(F_w \| F_{sw}).$$
And hence
\begin{align}
    & \kl(F_w \| g) - \kl(F_w \| F_{sw}) \ge 0 
    \nonumber \\ \Rightarrow & \int_{\cX} F_w(x) \log  \frac{F_w(x)}{g(x)} d x - \int_{\cX} F_w(x) \log \frac{F_w(x)}{F_{sw}(x)} d x \ge 0 
    \nonumber \\ \Rightarrow & \underbrace{\int_{\cX} F_w(x) \log \frac{F_{sw}(x)}{g(x)} d x}_{Q_2} \ge 0. \nonumber
\end{align}
Therefore, 
\begin{align}
Q_2 = \kl(F_{sw} \| g) + Q_1 
\ge 0. \label{eqn:cross-term-inequality-1}
\end{align}

Now, fix $t \in (0,1)$, and consider the functions 
\begin{align} 
& w(x) = (F_{sw}(x)) \cdot \left(\frac{F^\star(x)}{F_{sw}(x)} \right)^t, \label{eq:def_w_t} \\
& w'(x)= F_{sw}(x) + t\left( F^\star(x) - F_{sw}(x) \right).\label{eq:def_w'_t}
\end{align}

It is clear that $w(x)>0$, $w'(x)>0$.
And according to~\cref{lemma:taylor}, we have
\begin{align*}
    w(x)=F_{sw}(x) \cdot \left[ 1+t \log{\frac{F^\star(x)}{F_{sw}(x)}} + \frac{1}{2}t^2 \left(\log{\frac{F^\star(x)}{F_{sw}(x)}}\right)^2 \left( \frac{F^\star(x)}{F_{sw}(x)} \right)^\xi \right],
\end{align*}
where $\xi \in (0,t)$.
It means that
\begin{align} \label{diff_w_w'}
    w'(x)-w(x) & = \underbrace{t \cdot F_{sw}(x) \left(\frac{F^\star(x)}{F_{sw}(x)} - 1 - \log \frac{F^\star(x)}{F_{sw}(x)} \right)}_{\cO (t)} - \underbrace{t^2 \cdot \frac{F_{sw}(x)}{2} \left( \log \frac{F^\star(x)}{F_{sw}(x)} \right)^2 \left( \frac{F^\star(x)}{F_{sw}(x)} \right)^\xi}_{\cO (t^2)}.
\end{align}
And $\frac{F^\star(x)}{F_{sw}(x)} - 1 - \log \frac{F^\star(x)}{F_{sw}(x)}>0$. It means that as $t \to 0^+$, $w'(x)-w(x)>0$ and $w'(x)-w(x) = \cO(t)$.

Define
\begin{align}
    \phi(x) & \triangleq \log \frac{F_{sw}(x)}{w'(x)} - \log \frac{F_{sw}(x)}{w(x)} \label{eq:def_phi_x_t}
    \\ & = \log \frac{w(x)}{w'(x)} \nonumber
    \\ & = \log \left( 1-\frac{w'(x)-w(x)}{w'(x)} \right) \nonumber
    \\ & = -\underbrace{\frac{w'(x)-w(x)}{w'(x)}}_{\cO (t)} - \frac{1}{2}\underbrace{\left(\frac{w'(\zeta)-w(\zeta)}{w'(\zeta)} \right)^2}_{\cO (t^2)}, \tag{\cref{lemma:taylor_log}}
\end{align}
where $\zeta$ is between $0$ and $\frac{w'(x)-w(x)}{w'(x)}$.
As $t \to 0^+$, we know that $\phi(x) \le 0$ and $\phi(x) = \cO (t)$.

Combining~\eqref{eq:def_w_t} with~\eqref{eq:def_phi_x_t} and we have:
\begin{align}
    & \log \frac{F_{sw}(x)}{w(x)} = t \log \frac{F_{sw}(x)}{F^\star(x)}, \nonumber \\
    \text{and} \quad  & \log \frac{F_{sw}(x)}{w'(x)} = t \log \frac{F_{sw}(x)}{F^\star(x)} + \phi(x). \label{equation:w_w}
\end{align}

Note that $F_{sw} \in V_s$, which is a convex set (\Cref{claim:Vs-convex}).
Also, $F^\star \in V_s$ (Realizability).
Therefore, according to the definition of convexity,
we have $w' \in V_s$. 
Hence, substituting $w'$ for $g$ in \eqref{eqn:cross-term-inequality-1} and consider \cref{equation:w_w}, we get

\begin{align}
    & Q_2 = \underbrace{\int_{\cX} F_{sw}(x) \log \frac{F_{sw}(x)}{g(x)} d x}_{\kl(F_{sw} \| g)} + \underbrace{\int_{\cX} (F_w(x) - F_{sw}(x)) \log \frac{F_{sw}(x)}{g(x)} d x}_{Q_1} \ge 0, \nonumber \\
    \Rightarrow \quad & \kl(F_{sw} \| w') + \int_{\cX} (F_w(x) - F_{sw}(x)) \cdot \phi(x) d x + \nonumber \\ & \hspace{5.5cm} t \int_{\cX} (F_w(x) - F_{sw}(x)) \log \frac{F_{sw}(x)}{F^\star(x)} d x \ge 0. \label{discriminant}
\end{align}

Here, we address these three components individually.
Our goal is to show that the first term is $\cO(t^2)$, while the second term, which is $\cO(t)$, is negative. Considering these collectively, the third term, also $\cO(t)$, must be positive.

\paragraph{The first term.}
Denote $u(x)=F^\star(x) - F_{sw}(x)$. Note that 
\begin{align*}
\kl(F_{sw} \| w') & = \int_{\cX} (F_{sw}(x)) \log \frac{F_{sw}(x)}{F_{sw}(x) + t \cdot u(x)} d x 
\\ & = \int_{\cX} (F_{sw}(x)) \log \left( 1-\frac{t \cdot u(x)}{F_{sw}(x) + t \cdot u(x)} \right) d x 
\\ & = -t \int_{\cX} \frac{F_{sw}(x) \cdot u(x)}{F_{sw}(x) + t \cdot u(x)} d x - \frac{1}{2}t^2 \int_{\cX} \frac{F_{sw}(x) \cdot (u(\zeta'))^2}{(w'(\zeta'))^2} d x \tag{\cref{lemma:taylor_log}}
\\ & = -t \int_{\cX} \left[ u(x)- \frac{t \cdot (u(x))^2}{F_{sw}(x) + t \cdot u(x)} \right] d x - \frac{1}{2}t^2 \int_{\cX} \frac{F_{sw}(x) \cdot (u(\zeta'))^2}{(w'(\zeta'))^2} d x
\\ & = -t \underbrace{\int_{\cX} u(x) d x}_{0} + t^2 \int_{\cX} \frac{(u(x))^2}{F_{sw}(x) + t \cdot u(x)} d x - t^2 \int_{\cX} \frac{F_{sw}(x) \cdot (u(\zeta'))^2}{2(w'(\zeta'))^2} d x
\\ & = t^2 \int_{\cX} \left[ \frac{(u(x))^2}{F_{sw}(x) + t \cdot u(x)}-\frac{F_{sw}(x) \cdot (u(\zeta'))^2}{2(w'(\zeta'))^2} \right] d x
\end{align*}
where $\zeta'$ is between $0$ and $\frac{t \cdot u(x)}{F_{sw}(x) + t \cdot u(x)}$.
Therefore, taking the limit as $t \to 0^+$, we get that $\kl(F_{sw} \| w') = \cO (t^2)$. 

\paragraph{The second term.}

Recall~\cref{diff_w_w'} that
\begin{align*}
    & w'(x)-w(x) = t \cdot F_{sw}(x) \left(\frac{F^\star(x)}{F_{sw}(x)} - 1 - \log \frac{F^\star(x)}{F_{sw}(x)} \right) + \cO (t^2), \\ & w'(x)= F_{sw}(x) \left( 1 + t\left( \frac{F^\star(x)}{F_{sw}(x)} - 1 \right) \right).
\end{align*}

Therefore,
\begin{align*}
    \frac{w'(x)-w(x)}{w'(x)} & =\frac{t \left(\frac{F^\star(x)}{F_{sw}(x)} - 1 - \log \frac{F^\star(x)}{F_{sw}(x)} \right) + \cO (t^2)}{1 + t\left( \frac{F^\star(x)}{F_{sw}(x)} - 1 \right)} 
    \\ & = \left[ t \left(\frac{F^\star(x)}{F_{sw}(x)} - 1 - \log \frac{F^\star(x)}{F_{sw}(x)} \right) + \cO (t^2) \right] \cdot \left[ 1-t\left( \frac{F^\star(x)}{F_{sw}(x)} - 1 \right) + \cO (t^2) \right] \tag{\cref{lemma:taylor_frac}}
    \\ & = t \left(\frac{F^\star(x)}{F_{sw}(x)} - 1 - \log \frac{F^\star(x)}{F_{sw}(x)} \right) + \cO (t^2).
\end{align*}

In other words, the second term in \eqref{discriminant} is
\begin{align*}
& \int_{\cX} (F_w(x) - F_{sw}(x)) \cdot \phi(x) d x 
\\ = & -\int_{\cX} (F_w(x) - F_{sw}(x)) \cdot \left(\frac{w'(x)-w(x)}{w'(x)} + \cO (t^2)\right) d x \tag{\cref{eq:def_phi_x_t}} 
\\ = & \int_{\cX} (F_{sw}(x) - F_w(x)) \cdot \left(\frac{w'(x)-w(x)}{w'(x)} \right) d x + \cO (t^2) 
\\ = & t\int_{\cX} (F_{sw}(x) - F_w(x)) \left(\frac{F^\star(x)}{F_{sw}(x)} - 1 - \log \frac{F^\star(x)}{F_{sw}(x)} \right) d x + \cO (t^2) 
\\ = & t \cdot \wis(F^\star \| F_{sw}) + \cO (t^2),
\end{align*}
where the weight function of the weighted Itakura–Saito Divergence is $w=F_{sw} - F_w$.
Therefore, as $t \to 0^+$, the second term in~\eqref{discriminant} is of the order $\cO (t)$.
If $\wis(F^\star \| F_{sw}) \le 0$, the second term in~\eqref{discriminant} will be non-positive (otherwise, it will be positive).

\paragraph{The third term.}
Taking the limit as $t \to 0^+$ in~\eqref{discriminant}:
\begin{align*} 
\underbrace{\kl(F_{sw} \| w')}_{\cO (t^2)} + \underbrace{\int_{\cX} (F_w(x) - F_{sw}(x)) \cdot \phi(x) d x}_{\cO (t)} + \underbrace{t \int_{\cX} (F_w(x) - F_{sw}(x)) \log \frac{F_{sw}(x)}{F^\star(x)} d x}_{\cO (t)} \ge 0.
\end{align*}

If the middle term is non-positive, the last term should be non-negative:
\begin{align}
    \int_{\cX} (F_w(x) - F_{sw}(x)) \log \frac{F_{sw}(x)}{F^\star(x)} d x \ge 0. \label{eqn:cross-term-inequality-2}
\end{align}
Substituting $F^\star$ for $g$ in \eqref{eqn:square-expansion}, and using \eqref{eqn:cross-term-inequality-2}, we obtain the desired result
\begin{align*}
    \dist(F_{sw}, F^\star) \le \dist(F_w, F^\star) - \dist(F_w, F_{sw}).
\end{align*}

Else, if the middle term is positive (i.e,, $\wis(F^\star \| F_{sw}) \le 0$ is not satisfied), the last term may also be non-negative, which means that $\dist(F_{sw}, F^\star) \le \dist(F_w, F^\star) - \dist(F_w, F_{sw})$ may also hold.

\end{proof}

The following tools used in the above proof can be proved by Taylor expansion.

\begin{claim} \label{lemma:taylor}
For $t,x \in \R_+$, there holds
$$x^{t} = 1 + t\log x + \frac{1}{2}t^2(\log x)^2 x^\xi,$$
where $\xi \in (0,t)$.
\end{claim}

\begin{claim} \label{lemma:taylor_log}
For $x \in (0,1)$, there holds:
$$\log (1-x) = -x - \frac{1}{2}\zeta^2,$$
where $\zeta \in (0,x)$.
\end{claim}

\begin{claim} \label{lemma:taylor_frac}
For $x \in (-1,1)$, there holds:
$$\frac{1}{1+x}=1-x+\epsilon^2,$$
where $\epsilon$ is between $0$ and $x$.
\end{claim}

\subsection{Proof of~\cref{thm:non-realizable-finite-samples-main}} \label{proof_non-realizable-main}

\paragraph{Proof sketch of~\citet{charikar2024quantifying}.}
They apply the proof technique from their Theorem 1 to different variables and obtain several inequalities.
Subsequently, leveraging the triangle inequality for the $\ell_2$-norm and a uniform convergence argument, they establish the desired result.

\paragraph{Proof sketch of ours.}
Extending the proof framework from Theorem 2 in~\citep{charikar2024quantifying} is also non-trivial. 
Specifically, the absence of a triangle inequality for KL divergence necessitates an alternative approach. 
To address this, we decompose the relevant terms in a manner analogous to the triangle inequality and exhaustively demonstrate that each of the three resulting remainder terms asymptotically converges to zero.

\paragraph{Notations.}
For a clear presentation, let
\begin{align*}
    A &= \dist(F_s, F_{sw})\\
    B &=\dist(F_{sw}, F_w) \\
    C &= \dist(F_s, F_w)\\
    D &= \dist(F^\star, F_s) =\eps\\
    E &= \dist(F^\star, F_{sw}) \\
    F &= \dist(F^\star, F_w) \\
    G &= \dist(F^\star, \hat{F}_{sw}) \\
    H &= \dist(\hat{F}_{sw}, F_{sw}) \\
    I &= \dist(\hat{F}_{sw}, F_w).
\end{align*}

Now we start the proof of~\cref{thm:non-realizable-finite-samples-main}.
A uniform convergence result and two claims used in the proof (\cref{lem:uniform-convergence}, \cref{claim_xlnx} and~\cref{claim_xlnx_reverse}) are provided at the end of the proof.

\begin{proof}
Non-realizable weak-to-strong generalization where $F^\star \notin V_s$, and we use a finite sample to perform weak-to-strong supervision. 
Note that by virtue of the range of $f^\star, f_w$ and all functions in $\cF$ being absolutely bounded, and $\dist$ is also bounded.

Due to $F^\star \notin V_s$, we replace $F^\star$ with $F_s$ in the final step of proof of~\Cref{thm:realizable-main}, we obtain
\begin{align}
    C \ge A + B. \label{eqn:1-main}
\end{align}

Notice that
\begin{align}
    & E = A + D - \underbrace{\int_\cX (F^\star(x)-F_s(x))\log \frac{F_{sw}(x)}{F_s(x)} d x}_{t_1}, \label{eqn:2-main}\\ 
    & F = C + D - \underbrace{\int_\cX (F^\star(x)-F_s(x))\log \frac{F_w(x)}{F_s(x)} d x}_{t_2}, \label{eqn:2-1-main} \\
    & G = E - H - \underbrace{\int_\cX (\hat{F}_{sw}(x)-F^\star(x))\log \frac{F_{sw}(x)}{\hat{F}_{sw}(x)} d x}_{t_3} \label{eqn:2-2-main}. 
\end{align}

Combining \eqref{eqn:1-main} and \eqref{eqn:2-main}, we get
\begin{align}
    E \le C + D - B - t_1. \label{eqn:3-main}
\end{align}

By a uniform convergence argument (\Cref{lem:uniform-convergence-rev}), we have that with probability at least $1-\delta$ over the draw of $\{(x_1,y_1),\dots, (x_n,y_n)\}$ that were used to construct $\hat{F}_{sw}$,
\begin{align}
    I &\le B + \underbrace{\cO\left(\sqrt{\frac{\cC_{\cF_s}}{n}}\right)}_{t_4} + \underbrace{\cO\left(\sqrt{\frac{\log(1/\delta)}{n}}\right)}_{t_5}. \label{eqn:uc-main}
\end{align}

Combining \eqref{eqn:3-main} with \eqref{eqn:uc-main} and we have
\begin{align}
    E \le C + D - I - t_1 + t_4 + t_5. \label{eqn:4-main}
\end{align}

Combining \eqref{eqn:2-1-main} with \eqref{eqn:4-main} and we have
\begin{align}
    E \le F - I - t_1 + t_2 + t_4 + t_5. \label{eqn:5-1-main}
\end{align}

Combining \eqref{eqn:2-2-main} with \eqref{eqn:5-1-main} and we have
\begin{align}
    G \le F - I - H - t_1 + t_2 - t_3 + t_4 + t_5. \label{eqn:5-main}
\end{align}

We replace $F^\star$ with $\hat{F}_{sw}$ in the final step of proof of~\Cref{thm:realizable} and obtain:
\begin{align}
    I \ge H + B. \label{eqn:uc-projection-main}
\end{align}

Combining \eqref{eqn:uc-projection-main} with \eqref{eqn:uc-main} and we have
\begin{align}
    0 \le H \le t_4 + t_5 = \cO\left(\sqrt{\frac{\cC_{\cF_s}}{n}}\right) + \cO\left(\sqrt{\frac{\log(1/\delta)}{n}}\right). \label{eqn:6-main}
\end{align}

Combining \eqref{eqn:6-main} with \eqref{eqn:5-main} and we have
\begin{align}
    G \le F - I - t_1 + t_2 - t_3 + t_4 + t_5. \label{eqn:7-main}
\end{align}

While $t_4$ and $t_5$ are known in~\eqref{eqn:uc-main}, we analyze $t_1$, $t_2$ and $t_3$ one by one.

\paragraph{Deal with $t_1$.}

We know that
\begin{align}
    t_1 & = \int_\cX (F^\star(x)-F_s(x))\log \frac{F_{sw}(x)}{F_s(x)} d x. \nonumber
\end{align}
Using Pinsker's inequality and the fact that $\frac{F_{sw}(x)}{F_s(x)} \le \frac{1}{\gamma}$, we have
\begin{align} \label{def_t_1_ineq-main}
    |t_1| \le \frac{1}{\gamma} \int_{\cX}\left| F_s(x) - F^\star(x) \right| d x \le \frac{1}{\gamma} \sqrt{\frac{1}{2}\kl(F_s \| F^\star)} = \frac{1}{\gamma} \sqrt{\frac{1}{2}\varepsilon}.
\end{align}

Therefore,
\begin{align} \label{equation_t_1-main}
    |t_1| = \cO(\sqrt{\varepsilon}).
\end{align}

\paragraph{Deal with $t_2$.}
The proof for $t_2$ is similar for $t_1$.
In particular, replacing $F_{sw}$ with $F_w$ in the above and we can get 
\begin{align} \label{equation_t_2-main}
    |t_2| = O(\sqrt{\varepsilon}).
\end{align}

\paragraph{Deal with $t_3$.}

We know that
$$t_3 = \int_\cX (\hat{F}_{sw}(x)-F^\star(x))\log \frac{F_{sw}(x)}{\hat{F}_{sw}(x)} d x.$$

According to~\cref{lem:uniform-convergence-rev}, 
with probability at least $1-\delta$ over the draw of $(x_1,y_1),\dots,(x_n, y_n)$, we have
\begin{align} \label{t_3_proof_unif_conv-main}
    \left|d_{\cP}(\hat{F}_{sw},F_w) - d_{\cP}(F_{sw},F_w) \right| \le \cO\left(\sqrt{\frac{\cC_{\cF}}{n}}\right) + \cO\left(\sqrt{\frac{\log(1/\delta)}{n}}\right).
\end{align}

According to~\cref{claim_xlnx} and~\eqref{t_3_proof_unif_conv-main}, we have
\begin{align*}
    \left| \int_\cX \log \frac{F_{sw}(x)}{\hat{F}_{sw}(x)} d x \right| \le \cO\left(\sqrt{\frac{\cC_{\cF}}{n}}\right) + \cO\left(\sqrt{\frac{\log(1/\delta)}{n}}\right).
\end{align*}

Since $|\hat{F}_{sw}(x)-F^\star(x)|$ is upper bounded, there holds
\begin{align} \label{equation_t_3-main}
    |t_3| = \left| \int_\cX (\hat{F}_{sw}(x)-F^\star(x))\log \frac{F_{sw}(x)}{\hat{F}_{sw}(x)} d x \right| \le \cO\left(\sqrt{\frac{\cC_{\cF}}{n}}\right) + \cO\left(\sqrt{\frac{\log(1/\delta)}{n}}\right).
\end{align}

Therefore, combing~\eqref{equation_t_1-main}, \eqref{equation_t_2-main} and \eqref{equation_t_3-main}, we have
\begin{align} \label{t_1_2_3_ineq-main}
    |t_1| + |t_2| + |t_3| \le O(\sqrt{\varepsilon}) + \cO\left(\sqrt{\frac{\cC_{\cF}}{n}}\right) + \cO\left(\sqrt{\frac{\log(1/\delta)}{n}}\right).
\end{align}

Finally, combing~\eqref{eqn:uc-main} and~\eqref{eqn:7-main} with \eqref{eqn:6-main} and~\eqref{t_1_2_3_ineq-main}, we get the result:
\begin{align*}
\dist(F^\star, \hat{F}_{sw}) \le \dist(F^\star, F_w) - \dist(\hat{F}_{sw}, F_w) + O(\sqrt{\eps}) + \cO\left(\sqrt{\frac{\cC_{\cF}}{n}}\right) + \cO\left(\sqrt{\frac{\log(1/\delta)}{n}}\right),
\end{align*}
where in the last inequality, we instantiate asymptotics with respect to $\eps \to 0$ and $n \to \infty$.

\end{proof}

Here are some tools used in the above proof.

\begin{lemma}[Uniform convergence (forward KL loss)]
\label{lem:uniform-convergence}
Let $(x_1,y_1),\dots,(x_n, y_n)$ be an i.i.d. training sample, where each $x_i \sim \cP$ and $y_i = F_w(x_i)$ for a target function $F_w$. For a fixed strong model representation $h_s$, let 
\begin{align*}
    & f_{sw} = \argmin_{f \in \cF_{s}}\; \dist(F_w, f \circ h_s),
    \\ & \hat{f}_{sw} = \argmin_{f \in \cF_s} \disthat(F_w, f \circ h_s).
\end{align*}
Assume that the range of $F_w$ and functions in $\cF_s$ is absolutely bounded. Then, with probability at least $1-\delta$ over the draw of $(x_1,y_1),\dots,(x_n, y_n)$, we have
\begin{align*}
    \left|d_{\cP}(F_w, \hat{F}_{sw}) - d_{\cP}(F_w, F_{sw}) \right| \le \cO\left(\sqrt{\frac{\cC_{\cF_s}}{n}}\right) + \cO\left(\sqrt{\frac{\log(1/\delta)}{n}}\right),
\end{align*}
where $\cC_{\cF_s}$ is a constant capturing the complexity of the function class $\cF_s$.
\end{lemma}

\begin{proof}
The proof is strongly motivated by lemma 4 in~\citet{charikar2024quantifying}.

Note that
\begin{multline}
    \label{eqn:risk-decomposition}
    d_{\cP}(F_w, \hat{F}_{sw}) - d_{\cP}(F_w, F_{sw}) = \underbrace{d_{\cP}(F_w, \hat{F}_{sw}) - \hat{d}_{\cP}(F_w, \hat{F}_{sw})}_{a} + \\ \underbrace{\hat{d}_{\cP}(F_w, \hat{F}_{sw}) -  \hat{d}_{\cP}(F_w, F_{sw})}_{b} + \underbrace{\hat{d}_{\cP}(F_w, F_{sw}) - d_{\cP}(F_w, F_{sw})}_{c}.
\end{multline}
By the definition of $\hat{f}_{sw}$, the second term $b\le 0$ in \eqref{eqn:risk-decomposition}. 
Therefore,
\begin{align} \label{proof:ineq:conv}
    \left| d_{\cP}(F_w, \hat{F}_{sw}) - d_{\cP}(F_w, F_{sw}) \right| \le |a| +|c|.
\end{align}
The terms $a$ and $c$ measure the difference between the empirical risk and true population risk, and can be controlled by a standard uniform convergence argument.

Let $S = \{(x_1,y_1),\dots,(x_n,y_n)\}$, where $x_i \sim \cP$ and $y_i = F_w(x_i)$. 
According to statistical learning theory literature~\citep{bartlett2002rademacher}, it first holds that with probability at least $1-\delta$,
\begin{align*}
    \sup_{f \in \cF_s} |\hat{d}_{\cP}(F_w, f \circ h_s) - d_{\cP}(F_w, f \circ h_s)| &\le O \left( \mathcal{R}_n (l(\cF_s))\right) + \cO\left(\sqrt{\frac{\log(1/\delta)}{n}}\right),
\end{align*}
where $\mathcal{R}_n (l(\cF_s))$ is the \textit{Rademacher complexity} of the loss class of $\cF_s$:
\begin{align*}
    \mathcal{R}_n (l(\cF_s)) &= \E_{S} \E_{\eps_i \sim \{-1,1\}} \sup_{f \in \cF_s} \frac{1}{n}\sum_{i=1}^n \eps_i \cdot \ell(f \circ h_s(x_i), y_i).
\end{align*}
Notice again that the model output space $\cY = \{ y \in \R| 0 < \gamma \le y \le 1, \gamma>0 \}$.
We can then use the assumption that the range of $F_w$ and $\cF_s$ is absolutely bounded, which implies that $\ell$ is both bounded and Lipschitz in both arguments. 
This allows us to use the contraction principle in Theorem 4.12 from~\citet{ledoux2013probability} so as to move from the Rademacher complexity of the loss class $l(\cF_s)$ to that of $\cF_s$ itself, and claim that with probability at least $1-\delta$,
\begin{align}
    \sup_{f \in \cF_s} |\hat{d}_{\cP}(F_w, f \circ h_s) - d_{\cP}(F_w, f \circ h_s)| & \le O \left( \mathcal{R}_n (\cF_s)\right) + \cO\left(\sqrt{\frac{\log(1/\delta)}{n}}\right) \label{eqn:rademacher-cxty-bound}
\end{align}
Finally, the Rademacher complexity $\mathcal{R}_n (\cF_s)$ can be upper bounded by a quantity known as the \textit{worst-case Gaussian complexity} of $\cF_s$; in any case, for a majority of parametric function classes $\cF_s$, this quantity scales as $\sqrt{\frac{\cC_{\cF_s}}{n}}$~\citep{bartlett2002rademacher}, where $\cC_{\cF_s}$ is a constant capturing the inherent complexity of $\cF_s$. 
Plugging this into \eqref{eqn:rademacher-cxty-bound} and considering $f=\hat{f}_{sw}$ or $f=f_{sw}$ in this inequality, we have
\begin{align*}
    \underbrace{\left| \hat{d}_{\cP}(F_w, \hat{F}_{sw}) - d_{\cP}(F_w, \hat{F}_{sw}) \right|}_{|a|} \le \cO\left(\sqrt{\frac{\cC_{\cF_s}}{n}}\right) + \cO\left(\sqrt{\frac{\log(1/\delta)}{n}}\right), \\
    \underbrace{\left| \hat{d}_{\cP}(F_w, F_{sw}) - d_{\cP}(F_w, F_{sw}) \right|}_{|c|} \le \cO\left(\sqrt{\frac{\cC_{\cF_s}}{n}}\right) + \cO\left(\sqrt{\frac{\log(1/\delta)}{n}}\right).
\end{align*}

Finally, substitute it into~\cref{proof:ineq:conv} and we can obtain the desired bound.
\end{proof}

\begin{lemma}[Uniform convergence (reverse KL loss)]
\label{lem:uniform-convergence-rev}
Let $(x_1,y_1),\dots,(x_n, y_n)$ be an i.i.d. training sample, where each $x_i \sim \cP$ and $y_i = F_w(x_i)$ for a target function $F_w$. For a fixed strong model representation $h_s$, let 
\begin{align*}
    & f_{sw} = \argmin_{f \in \cF_{s}}\; \dist(f \circ h_s, F_w),
    \\ & \hat{f}_{sw} = \argmin_{f \in \cF_s} \disthat(f \circ h_s, F_w).
\end{align*}
Assume that the range of $F_w$ and functions in $\cF_s$ is absolutely bounded. Then, with probability at least $1-\delta$ over the draw of $(x_1,y_1),\dots,(x_n, y_n)$, we have
\begin{align*}
    \left|d_{\cP}(\hat{F}_{sw}, F_w) - d_{\cP}(F_{sw}, F_w) \right| \le \cO\left(\sqrt{\frac{\cC_{\cF_s}}{n}}\right) + \cO\left(\sqrt{\frac{\log(1/\delta)}{n}}\right),
\end{align*}
where $\cC_{\cF_s}$ is a constant capturing the complexity of the function class $\cF_s$.
\end{lemma}

\begin{proof}
    Swap the order of the two elements in $\dist(\cdot, \cdot)$ and $\hat{d}_\cP(\cdot, \cdot)$ in the proof of~\cref{lem:uniform-convergence} and we can prove the result.
\end{proof}

\begin{claim} \label{claim_xlnx}
Let $f(x), g(x) \in [\gamma,1]$ where $\gamma>0$. If there exists $\xi>0$ such that $\int_{\cX} \left|f(x)-g(x) \right| d x \le \xi$,
then there holds
$$\int_{\cX} \left| \log f(x)- \log g(x) \right| d x \le \frac{1}{\gamma}\xi.$$
\end{claim}

\begin{proof}

Using the property of the function $\phi(x)=\log x$ (as shown in~\cref{fig:function_xlogx}):
if $x \in (0,1]$, then the slope of a line with any two points on the function $\phi(x)$ is bounded.
\begin{figure}[ht]
    \centering
    \includegraphics[width=0.4\linewidth]{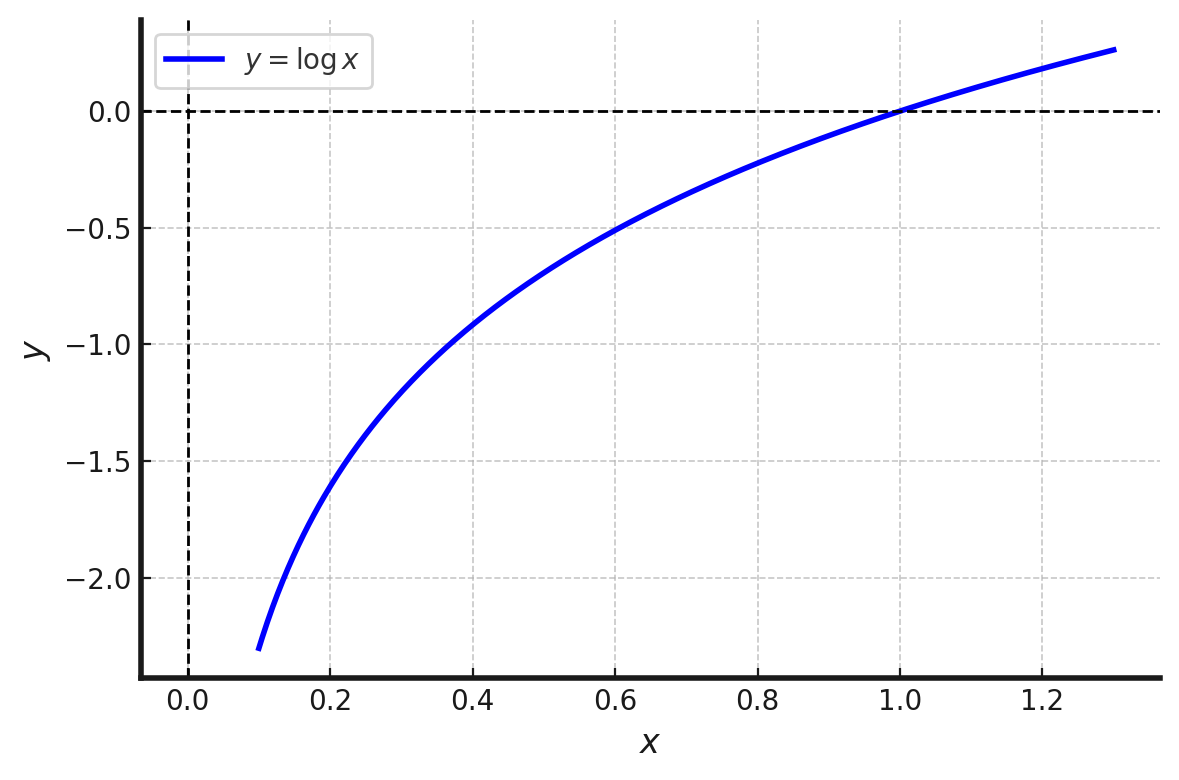}
    \caption{The function $\phi(x)= \log x$.}
    \label{fig:function_xlogx}
\end{figure}

In particular, we have
\begin{align*}
    & \int_{\cX} \left|\log f(x)-\log g(x) \right| d x \nonumber \\ = & \int_{\cX} \left|\frac{\log f(x)-\log g(x) }{f(x)-g(x)}\right| |f(x)-g(x)| d x \nonumber \\ \le & \frac{1}{\gamma} \int_{\cX} \left|f(x)-g(x) \right| d x \nonumber \\ \le & \frac{1}{\gamma} \xi.
\end{align*}
\end{proof}

\begin{claim} \label{claim_xlnx_reverse}
Let $f(x), g(x) \in [\gamma,1]$ where $\gamma>0$. If there exists $\xi>0$ such that $\int_{\cX} \left| \log f(x) - \log g(x) \right| d x \le \xi$,
then there holds
$$\int_{\cX} \left| f(x)- g(x) \right| d x \le \xi.$$
\end{claim}

\begin{proof}

Using the property of the function $\phi(x)=\log x$ (as shown in~\cref{fig:function_xlogx}):
if $x \in (0,1]$, then the slope of a line with any two points on the function $\phi(x)$ is bounded.

In particular, we have
\begin{align*}
    & \int_{\cX} \left|f(x)-g(x) \right| d x \nonumber \\ = & \int_{\cX} \left|\frac{\log f(x)-\log g(x)}{f(x)-g(x)}\right|^{-1} |\log f(x)-\log g(x)| d x \nonumber \\ \le & \int_{\cX} \left|\log f(x)-\log g(x) \right| d x \nonumber \\ \le & \xi.
\end{align*}
\end{proof}

\subsection{Extension of~\cref{thm:non-realizable-finite-samples-main}} \label{proof_non-realizable}

This additional theoretical result follows~\cref{thm:realizable} in~\cref{proof_theorem_1}.

\begin{corollary} \label{thm:non-realizable-finite-samples}
Given $F^\star$, $F_w$ and $F_{sw}$ defined above.
Let $\dist$ be the output distribution divergence and consider W2SG in~\cref{eqn:erm} using forward KL divergence loss.
Consider $\cF_s$ that satisfies~\cref{convex_set}.
If the weighted Itakura–Saito divergence $\wis(F^\star \| F_{sw}) \le 0$ with the weight function $w=F_{sw} - F_w$, then we have that with probability at least $1-\delta$ over the draw of $n$ i.i.d. samples,
\begin{align} 
\dist(\hat{F}_{sw}, F^\star) \le \dist(F_w, F^\star) - \dist(F_w, \hat{F}_{sw}) + \cO(\sqrt{\eps}) +  \cO\left(\sqrt{\frac{\cC_{\cF_s}}{n}}\right) + \cO\left(\sqrt{\frac{\log(1/\delta)}{n}}\right),
\end{align}
where $\cC_{\cF_s}$ is a constant capturing the complexity of the function class $\cF_s$, and the asymptotic notation is with respect to $\eps \to 0, n \to \infty$.
\end{corollary}

\begin{proof}

For a clear presentation, let
\begin{align*}
    A &= \dist(F_{sw}, F_s)\\
    B &=\dist(F_w, F_{sw}) \\
    C &= \dist(F_w, F_s)\\
    D &= \dist(F_s, F^\star) =\eps\\
    E &= \dist(F_{sw}, F^\star) \\
    F &= \dist(F_w, F^\star) \\
    G &= \dist(\hat{F}_{sw}, F^\star) \\
    H &= \dist(F_{sw}, \hat{F}_{sw}) \\
    I &= \dist(F_w, \hat{F}_{sw}).
\end{align*}

The main proof idea follows~\cref{proof_non-realizable-main}.
Specifically, we first replace $F^\star$ with $F_s$ in the final step of proof of~\Cref{thm:realizable}, we obtain
\begin{align}
    C \ge A + B. \label{eqn:1}
\end{align}

Notice that
\begin{align}
    & E = A + D - \underbrace{\int_\cX (F_{sw}(x)-F_s(x))\log \frac{F^\star(x)}{F_s(x)} d x}_{t_1}, \label{eqn:2}\\ 
    & F = C + D - \underbrace{\int_\cX (F_w(x)-F_s(x))\log \frac{F^\star(x)}{F_s(x)} d x}_{t_2}, \label{eqn:2-1} \\
    & G = E - H - \underbrace{\int_\cX (\hat{F}_{sw}(x)-F_{sw}(x))\log \frac{F^\star(x)}{F_{sw}(x)} d x}_{t_3} \label{eqn:2-2}. 
\end{align}

Combining \eqref{eqn:1} and \eqref{eqn:2}, we get
\begin{align}
    E \le C + D - B - t_1. \label{eqn:3}
\end{align}

According to~\Cref{lem:uniform-convergence}, we have that with probability at least $1-\delta$ over the draw of $\{(x_1,y_1),\dots, (x_n,y_n)\}$ that were used to construct $\hat{F}_{sw}$,
\begin{align}
    I &\le B + \underbrace{\cO\left(\sqrt{\frac{\cC_{\cF_s}}{n}}\right)}_{t_4} + \underbrace{\cO\left(\sqrt{\frac{\log(1/\delta)}{n}}\right)}_{t_5}. \label{eqn:uc}
\end{align}

Combining \eqref{eqn:3} with \eqref{eqn:uc} and we have
\begin{align}
    E \le C + D - I - t_1 + t_4 + t_5. \label{eqn:4}
\end{align}

Combining \eqref{eqn:2-1} with \eqref{eqn:4} and we have
\begin{align}
    E \le F - I - t_1 + t_2 + t_4 + t_5. \label{eqn:5-1}
\end{align}

Combining \eqref{eqn:2-2} with \eqref{eqn:5-1} and we have
\begin{align}
    G \le F - I - H - t_1 + t_2 - t_3 + t_4 + t_5. \label{eqn:5}
\end{align}

We replace $F^\star$ with $\hat{F}_{sw}$ in the final step of proof of~\Cref{thm:realizable} (Recall the fact that $\hat{F}_{sw} \in V_s$ and~\eqref{eqn:cross-term-inequality-2}: substituting $\hat{F}_{sw}$ for $g$ in \eqref{eqn:square-expansion}, and using \eqref{eqn:cross-term-inequality-2}), we obtain:
\begin{align}
    I \ge H + B. \label{eqn:uc-projection}
\end{align}

Combining \eqref{eqn:uc-projection} with \eqref{eqn:uc} and we have
\begin{align}
    0 \le H \le t_4 + t_5 = \cO\left(\sqrt{\frac{\cC_{\cF_s}}{n}}\right) + \cO\left(\sqrt{\frac{\log(1/\delta)}{n}}\right). \label{eqn:6}
\end{align}

Combining \eqref{eqn:6} with \eqref{eqn:5} and we have
\begin{align}
    G \le F - I - t_1 + t_2 - t_3 + t_4 + t_5. \label{eqn:7}
\end{align}

While $t_4$ and $t_5$ are known in~\eqref{eqn:uc}, we analyze $t_1$, $t_2$ and $t_3$ one by one.


\paragraph{Deal with $t_1$.}
We know that
\begin{align}
    t_1 & = \int_\cX (F_{sw}(x)-F_s(x))\log \frac{F^\star(x)}{F_s(x)} d x. \nonumber
\end{align}
Using the fact that $\left|F_{sw}(x)-F_s(x)\right| \le 1$, we have
\begin{align} \label{def_t_1_ineq}
    |t_1| \le \int_{\cX} \left| \log \frac{F^\star(x)}{F_s(x)} \right| d x = \int_{\cX} \left| \log F_s(x) - \log F^\star(x) \right| d x.
\end{align}

According to Pinsker's inequality,
\begin{align} \label{output_diff}
    \int_{\cX}\left| F_s(x) - F^\star(x) \right| d x \le \sqrt{\frac{1}{2}\kl(F_s \| F^\star)} = \sqrt{\frac{1}{2}\varepsilon}.
\end{align}
Substitute $f(x)=F_s(x)$, $g(x)=F^\star(x)$ and $\xi = \sqrt{\frac{1}{2}\varepsilon}$ into~\cref{claim_xlnx} and recall~\eqref{def_t_1_ineq}, we have
\begin{align} \label{equation_t_1}
    |t_1| \le \frac{1}{\gamma}\sqrt{\frac{1}{2}\varepsilon} = O(\sqrt{\varepsilon}).
\end{align}

\paragraph{Deal with $t_2$.}
The proof for $t_2$ is similar for $t_1$.
In particular, replacing $F_{sw}$ with $F_w$ in the above and we can get 
\begin{align} \label{equation_t_2}
    |t_2| = O(\sqrt{\varepsilon}).
\end{align}

\paragraph{Deal with $t_3$.}

We know that
$$t_3 = \int_\cX (\hat{F}_{sw}(x)-F_{sw}(x))\log \frac{F^\star(x)}{F_{sw}(x)} d x.$$

According to~\cref{lem:uniform-convergence}, 
with probability at least $1-\delta$ over the draw of $(x_1,y_1),\dots,(x_n, y_n)$, we have
\begin{align} \label{t_3_proof_unif_conv}
    \left|d_{\cP}(F_w, \hat{F}_{sw}) - d_{\cP}(F_w, F_{sw}) \right| \le \cO\left(\sqrt{\frac{\cC_{\cF}}{n}}\right) + \cO\left(\sqrt{\frac{\log(1/\delta)}{n}}\right).
\end{align}
Notice that 
\begin{align} \label{t_3_proof}
H & = d_{\cP}(F_{sw}, \hat{F}_{sw}) \nonumber \\ & = d_{\cP}(F_w, F_{sw}) - d_{\cP}(F_w, \hat{F}_{sw}) + \int_\cX (F_w(x)+F_{sw}(x))\log \frac{F_{sw}(x)}{\hat{F}_{sw}(x)} d x.
\end{align}
Substitute~\eqref{eqn:6} and~\eqref{t_3_proof_unif_conv} into~\cref{t_3_proof} with the triangle inequality for absolute values, we get
\begin{align*}
    \left| \int_\cX (F_w(x)+F_{sw}(x))\log \frac{F_{sw}(x)}{\hat{F}_{sw}(x)} d x \right| \le \cO\left(\sqrt{\frac{\cC_{\cF}}{n}}\right) + \cO\left(\sqrt{\frac{\log(1/\delta)}{n}}\right)
\end{align*}
Since $\left| F_w(x)+F_{sw}(x) \right|$ is bounded, we have
$$\left| \int_\cX \left[ \log F_{sw}(x)- \log \hat{F}_{sw}(x) \right] d x \right| \le \cO\left(\sqrt{\frac{\cC_{\cF}}{n}}\right) + \cO\left(\sqrt{\frac{\log(1/\delta)}{n}}\right).$$

Using~\cref{claim_xlnx_reverse}, we have
$$\left| \int_\cX (\hat{F}_{sw}(x)-F_{sw}(x)) d x \right| \le \cO\left(\sqrt{\frac{\cC_{\cF}}{n}}\right) + \cO\left(\sqrt{\frac{\log(1/\delta)}{n}}\right).$$
Since $\left| \log \frac{F^\star(x)}{F_{sw}(x)} \right|$ is bounded, there holds
\begin{align} \label{equation_t_3}
    |t_3| = \left| \int_\cX (\hat{F}_{sw}(x)-F_{sw}(x)) \log \frac{F^\star(x)}{F_{sw}(x)} d x  \right| \le \cO\left(\sqrt{\frac{\cC_{\cF}}{n}}\right) + \cO\left(\sqrt{\frac{\log(1/\delta)}{n}}\right).
\end{align}

Therefore, combing~\eqref{equation_t_1}, \eqref{equation_t_2} and \eqref{equation_t_3}, we have
\begin{align} \label{t_1_2_3_ineq}
    |t_1| + |t_2| + |t_3| \le O(\sqrt{\varepsilon}) + \cO\left(\sqrt{\frac{\cC_{\cF}}{n}}\right) + \cO\left(\sqrt{\frac{\log(1/\delta)}{n}}\right).
\end{align}

Finally, combing~\eqref{eqn:uc} and~\eqref{eqn:7} with \eqref{eqn:6} and~\eqref{t_1_2_3_ineq}, we get the result:
\begin{align*}
\dist(\hat{F}_{sw}, F^\star) \le \dist(F_w, F^\star) - \dist(F_w, \hat{F}_{sw}) + O(\sqrt{\eps}) + \cO\left(\sqrt{\frac{\cC_{\cF}}{n}}\right) + \cO\left(\sqrt{\frac{\log(1/\delta)}{n}}\right),
\end{align*}
where in the last inequality, we instantiate asymptotics with respect to $\eps \to 0$ and $n \to \infty$.
\end{proof}

\section{Further Details and Results of Experiments}

\subsection{Training Details of Experiments in Language Models} \label{exp_llm_training_detail}

The dataset is randomly divided into three distinct subsets:
\begin{itemize}
    \item 4K samples (ground truth): They are used to fine-tune weak and strong base language models;
    \item 4K samples (held-out set): These labels are predicted by the weak model and used to provide weak supervision for training the strong model;
    \item 4K samples (the remaining): They are used for testing and evaluating the performance of all models.
\end{itemize}

When fine-tuning the weak-to-strong models, we follow~\citep{yang2024super} to set the batch size to $32$, learning rate to $10^{-5}$, \texttt{max\_seq\_len} to $512$.
The training epoch is set to 1 to avoid overfitting. All experiments are conducted on NVIDIA A100 80G.

\subsection{Weak-to-Strong Training Procedure in Synthetic Experiments} \label{appendix:syn_train}

We explore two methods to generate the weak and strong representations, which follows the experimental setting in~\citep{charikar2024quantifying}.

\begin{itemize}
    \item \textbf{Pre-training.} We begin by randomly sampling $T$ fine-tuning tasks $f_1^\star,\dots,f_T^\star \in \cF_s$. 
    For each $t \in \{ 1, \cdots, T \}$, we generate $N_r$ data $\{x^{(t)}_j\}_{j=1}^{N_r}$ where $x^{(t)}_j \sim \mathcal{P}$.
    Let the representations $h_w, h_s: \R^8 \to \R^{16}$ be 2-layer and 8-layer MLP with ReLU activations, respectively. And $h_w \in \cH_w$, $h_s \in \cH_s$.
    We obtain $h_w$ and $h_s$ via gradient descent on the representation parameters to find the minimizer of output distribution divergence loss. 
    Specifically, We use~\cref{def:kl_dist_emp} as the loss function on $T$ tasks:
    \begin{align} \label{eqn:expts-obtain-representations}
    h_l = \argmin_{h \in \cH_l} \ \frac{1}{T} \sum_{t=1}^T \disthat(f_t^\star \circ h, f_t^\star \circ h^\star),
    \end{align}
    where $l \in \{w, s\}$,  $T=10$, and $N_r=2000$.
    Additionally, the realizable setting (\Cref{thm:realizable}) is considered by explicitly setting $h_s = h^\star$, and only obtaining $h_w$ as above.
    \item \textbf{Perturbations.} As an alternative, we directly perturb the parameters of $h^\star$ to obtain the weak and strong representations. 
    Specifically, we add independent Gaussian noise $\mathcal{N}(0, \sigma_s^2)$ to every parameter in $h^\star$ to generate $h_s$. 
    Similarly, we perturb $h^\star$ with $\mathcal{N}(0, \sigma_w^2)$ to generate $h_w$. 
    To ensure the strong representation $h_s$ is closer to $h^\star$ than $h_w$, we set $\sigma_s=0.1$ and $\sigma_w=9$.
\end{itemize}

\paragraph{Weak Model Fine-tuning.} 
After obtaining $h_w$ and $h_s$, we fix these representations and train weak models on new fine-tuning tasks.
We randomly sample $M$ new fine-tuning tasks $f_1^\star,\dots,f_M^\star \in \cF_s$, and generate data $\{x^{(i)}_j\}_{j=1}^{N_f}$, where $x^{(i)}_j \sim \mathcal{P}$. 
For each task $i=\{1, \cdots, M \}$, the weak model is trained through:
\begin{align}
    \label{eqn:weak-model-finetuning}
    f^{(i)}_{w} &= \argmin_{f} \ \frac{1}{M} \sum_{i=1}^M \disthat(f_t^\star \circ h^\star, f \circ h_w),
\end{align}
where $M=100, N_f=2000$.
Here, the representation parameters $h_w$ are frozen, and $f^{(i)}_{w}$ is learned via gradient descent. 
Weak models are thus trained on true data.

\paragraph{Weak-to-Strong Supervision.} 
Using the trained weak models, we generate weakly labeled datasets for each fine-tuning task.
Specifically, for each $i \in \{ 1, \cdots, M \}$, we generate $\{\tilde{x}^{(i)}_j\}_{j=1}^{N_f}$ where $\tilde{x}^{(i)}_j \sim \mathcal{P}$. 
The strong models are then trained on these weakly labeled datasets by solving the following optimization problem using reverse KL divergence loss for each task $i \in \{ 1, \cdots, M \}$:
\begin{align}
    \label{eqn:weak-to-strong-supervision}
    f^{(i)}_{sw} &= \argmin_{f \in \cF} \ \disthat(f \circ h_s, f^{(i)}_{w} \circ h_w). 
\end{align}
At this stage, the weak-to-strong training procedure is complete.

\end{document}